\documentclass{article}

\usepackage[final]{corl_2026}   % camera-ready
\usepackage{amsmath,amssymb,bm,mathtools}
\usepackage{amsthm}
\usepackage{graphicx}
\usepackage{booktabs,multirow,array,tabularx}
\usepackage{siunitx}
\usepackage{enumitem}
\usepackage[capitalize,nameinlink]{cleveref}
\crefname{figure}{Figure}{Figures}
\Crefname{figure}{Figure}{Figures}
\crefname{table}{Table}{Tables}
\Crefname{table}{Table}{Tables}
\crefname{section}{Section}{Sections}
\Crefname{section}{Section}{Sections}
\crefname{subsection}{Section}{Sections}
\Crefname{subsection}{Section}{Sections}
\crefname{appendix}{Appendix}{Appendices}
\Crefname{appendix}{Appendix}{Appendices}
\usepackage{microtype}
\usepackage{xspace}
\usepackage{wrapfig}
\usepackage{subcaption}

\newtheorem{theorem}{Theorem}
\newtheorem{proposition}{Proposition}
\newtheorem{definition}{Definition}

\newcommand{\PFDS}{\textsc{PFDS}\xspace}
\newcommand{\PFDSOFF}{\textsc{PFDS\_OFF}\xspace}
\newcommand{\DSHC}{\textsc{DSHC}\xspace}
\newcommand{\DSHCOFF}{\textsc{DSHC\_OFF}\xspace}

\newcommand{\BranchMLP}{\textsc{Branch-MLP}\xspace}
\newcommand{\FlatMLP}{\textsc{Flat-MLP}\xspace}
\newcommand{\TactSet}{\textsc{TactSet}\xspace}
\newcommand{\etal}{\emph{et al.}\xspace}
\newcommand{\RR}{\mathbb{R}}
\newcommand{\Gdiag}{G_{\mathrm{diag}}}
\newcommand{\Gframe}{G_{\mathrm{frame}}}
\newcommand{\SN}{\mathfrak{S}_N}
\newcommand{\EE}{\mathbb{E}}
\newcommand{\Pool}{\mathrm{Pool}}

\title{%
S2M-Trek: From Single to Multi-Sphere Transport\\
  via Per-Frame Deep Sets on a Wheel-Legged Robot%
}

\author{%
  \bfseries Zong Chen$^{1,\dagger}$ \quad Xuebin Li$^{2}$ \quad Jinpeng Xiao$^{1}$ \quad Shaoyang Li$^{1}$ \\
  \bfseries Ben Liu$^{1}$ \quad Min Li$^{1}$ \quad Zhouping Yin$^{1}$ \quad Yiqun Li$^{1,*}$ \\
  \normalfont\small $^{1}$School of Mechanical Science and Engineering, Huazhong University of Science and Technology \\
  \small $^{2}$School of Mathematics, Harbin Institute of Technology \\
  $^{*}$Corresponding author \qquad $^{*}$\texttt{liyiqun@hust.edu.cn} \qquad $^{\dagger}$\texttt{skelon\_chan@hust.edu.cn}%
}

\begin{document}
\maketitle
\hypersetup{%
  pdftitle={S2M-Trek: From Single to Multi-Sphere Transport via Per-Frame Deep Sets on a Wheel-Legged Robot},
  pdfauthor={Zong Chen, Xuebin Li, Jinpeng Xiao, Shaoyang Li, Ben Liu, Min Li, Zhouping Yin, and Yiqun Li},
  pdfsubject={Proceedings of the 10th Conference on Robot Learning (CoRL 2026)}
}

\begin{center}
  \includegraphics[width=0.98\linewidth]{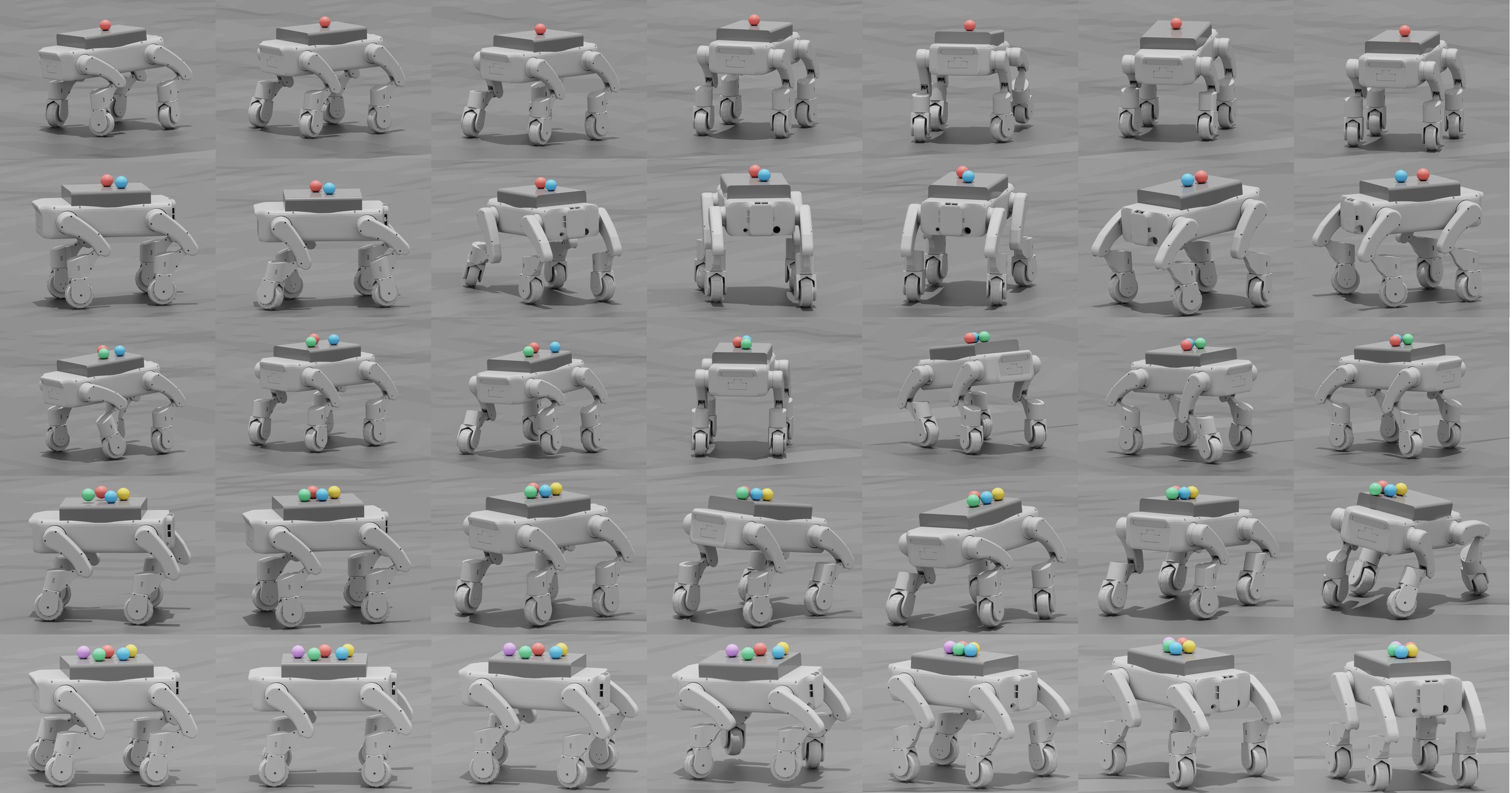}
  \captionof{figure}{Representative simulation rollouts for S2M-Trek. Each row shows a transport episode with $k{=}1$ through $k{=}5$ spheres, illustrating the open-plate task and the range of payload counts evaluated quantitatively in simulation.}
  \label{fig:sim-rollouts}
\end{center}
% ── abstract ────────────────────────────────────────────────────────────────
\begin{abstract}
Conventional object transport typically relies on containers, fences, or tethers to secure the payload. Here, a wheel-legged robot transports up to five identical, free-rolling steel spheres on an open back plate without physical restraints. Because the spheres are indistinguishable, their observation-slot assignments may change independently across history frames, creating a per-frame permutation symmetry that standard history-concatenation set encoders do not enforce. This mismatch becomes a bottleneck beyond two spheres: under the same PPO budget, flat and branch-wise MLPs plateau at or below two spheres, while a history-concatenation Deep Sets baseline (\DSHC) also stalls unless slot assignments are randomized during training. We introduce \textbf{Per-Frame Deep Sets (\PFDS)}, which pools objects invariantly within each frame before temporal fusion. \PFDS is $\Gframe$-invariant by construction, and we prove universal approximation of continuous $\Gframe$-invariant policies. In a $2{\times}2$ ablation over architecture and slot randomization, \PFDS reaches five spheres across all five seeds and achieves 100\% no-drop transport, whereas \DSHC relies on augmentation for reliable curriculum progression. We also distill the privileged \PFDS teacher into \TactSet, which observes only a $\Gframe$-invariant binary union contact map and retains 75\% no-drop performance with five spheres. Finally, we run the S2M-Trek deployment stack closed-loop on a physical wheel-legged robot for every sphere count $k{=}1,\cdots,5$. The measured tactile maps contain one connected component per sphere but no persistent identity, complementing the quantitative simulation results. Demonstrations are available at \href{https://drgon-li.github.io/PFDS}{drgon-li.github.io/PFDS}.

\end{abstract}

\keywords{loco-manipulation, permutation-invariant learning, tactile sensing}

% ============================================================================
\section{Introduction}
\label{sec:intro}
% ============================================================================

\begin{figure}[h] \centering \includegraphics[width=0.94\linewidth]{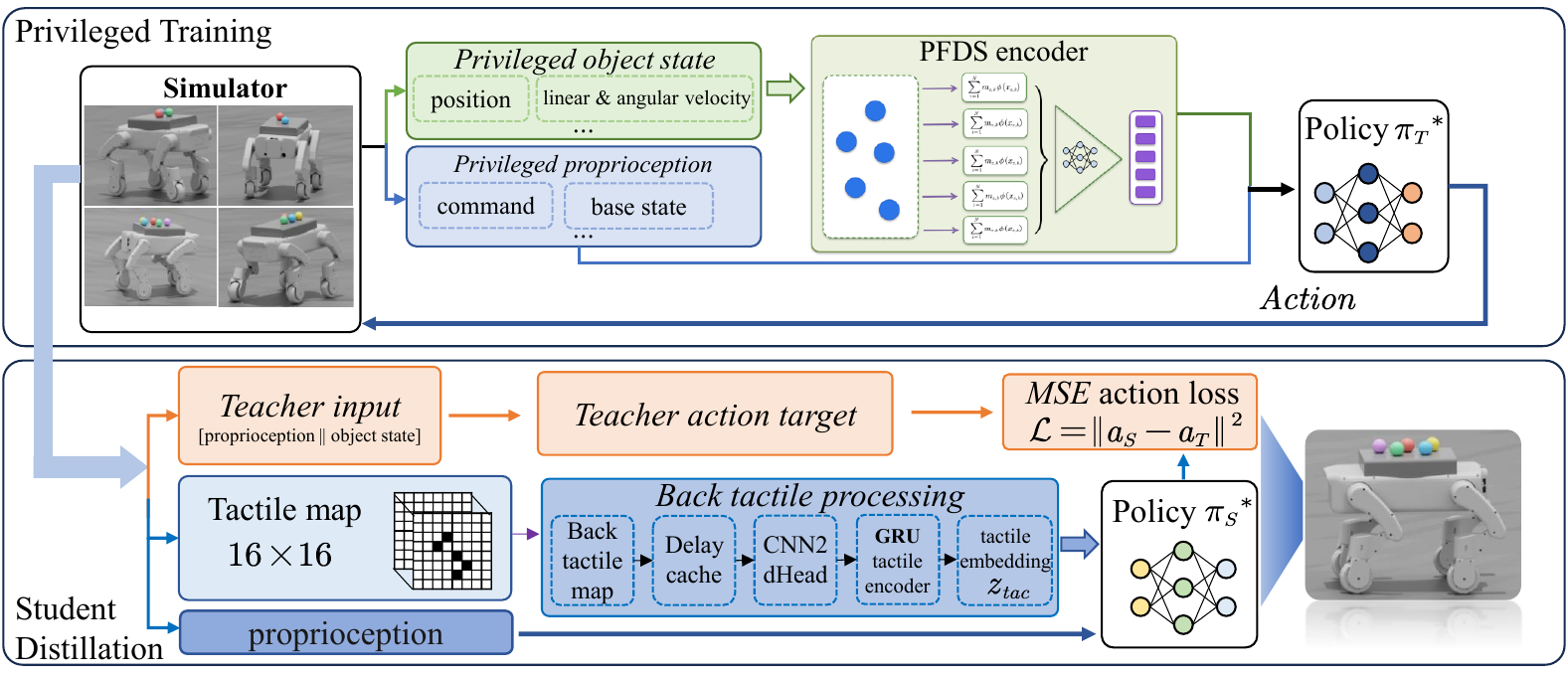} \caption{System overview of S2M-Trek (\textbf{top}: privileged teacher training; \textbf{bottom}: tactile-student distillation). The \PFDS teacher is trained with PPO in Isaac Lab~\citep{mittal2025isaaclab} using privileged ball-state observations, including position and linear/angular velocity. The teacher is then distilled into \TactSet via DAgger, replacing privileged object states with a binary union tactile contact map. The resulting deployment stack is qualitatively validated on the physical robot in \cref{sec:hardware}.}
\label{fig:framework}
\end{figure}

Wheel-legged robots have demonstrated versatile locomotion~\citep{hwangbo2019learning,rudin2022learning,miki2022robust,kumar2021rma}, and recent work extends them to single-object loco-manipulation~\citep{ji2023dribblebot,he2024visual,liu2024visual,lin2025locotouch}. Scaling to multiple freely rolling objects without fences, clamps, or persistent identities is not simply a matter of adding input channels: slot assignments can change independently across history frames, creating a per-frame permutation symmetry not addressed by prior loco-manipulation work. We study this setting in \emph{S2M-Trek} through large-scale simulation and qualitative physical validation; representative multi-sphere rollouts are shown in \cref{fig:sim-rollouts}, while all quantitative comparisons use simulation and hardware runs test closed-loop feasibility and tactile sensing.

\paragraph{The symmetry mismatch problem.}
Consider the multi-sphere observation tensor $X \in \RR^{H \times N \times d}$, where $H$ is the history length, $N$ is the maximum number of object slots, and $d$ is the per-object feature dimension. Because the spheres are identical, permuting their slot assignments within any individual history frame does not change the underlying physical state. The relevant symmetry is therefore the \emph{per-frame} product group
\begin{equation}
\label{eq:gframe}
\Gframe = (\SN)^H, \qquad
(g \cdot X)_{t,i} = x_{t,\pi_t^{-1}(i)},
\quad
g = (\pi_1,\cdots,\pi_H),
\end{equation}
rather than the \emph{diagonal} subgroup
$\Gdiag = \{(\pi,\cdots,\pi) : \pi \in \SN\} \subsetneq \Gframe$,
which applies the same permutation to every history frame and is the symmetry captured by history-concatenation set encoders~\citep{zaheer2017deepsets,lee2019settransformer}. The size ratio between these two groups is
$|\Gframe|/|\Gdiag| = (N!)^{H-1}$. For $N{=}5$ and $H{=}4$, this becomes $(5!)^3 = 1.728{\times}10^6$. In other words, the per-frame formulation identifies more than one million additional frame-wise relabelings as physically equivalent, whereas a $\Gdiag$-invariant encoder does not remove this redundancy.

\paragraph{Two paths to $\Gframe$ robustness.}
There are two ways to achieve robustness to $\Gframe$ permutations:
\begin{enumerate}[topsep=2pt,itemsep=0pt,leftmargin=*]
\item \textbf{Architectural (intrinsic):} Design the encoder so that $f(g \cdot X) = f(X)$ for all $g \in \Gframe$ by construction.
\item \textbf{Data augmentation (learned):} Retain a $\Gdiag$-invariant encoder, but sample a fresh per-frame slot permutation $\pi_t$ at each simulation step, thereby encouraging the policy to learn invariance over the $\Gframe$ orbit.
\end{enumerate}
Our experiments show that these two approaches are not equivalent. The architectural approach advances through the curriculum both with and without slot-permutation augmentation, whereas the $\Gdiag$ baseline stalls at the two-sphere stage when trained without augmentation.

\paragraph{Contributions.}
\begin{enumerate}[topsep=2pt,itemsep=0pt,leftmargin=*]

\item We formalize multi-sphere transport as a $\Gframe$-invariant policy learning problem and derive the quotient ratio $(N!)^{H-1}$ to quantify the representational gap left by standard history-based set encoders.

\item We propose \textbf{\PFDS} (Per-Frame Deep Sets), a minimal extension of standard Deep Sets that achieves $\Gframe$-invariance by pooling objects within each frame before temporal readout. We further establish its invariance and universal approximation properties.

\item We conduct an eight-encoder comparison and a $2{\times}2$ ablation over encoder architecture and slot-permutation augmentation, isolating the effects of architectural invariance and data augmentation. We further verify the reliability of \PFDS across five random seeds.

\item We distill the \PFDS teacher into \TactSet, a DAgger student that replaces privileged ball-state observations with a compact binary union contact map from a $16{\times}16$ tactile array. \TactSet retains 75\% no-drop performance with five spheres under the original simulator-asset generation, and 85\% in the matched evaluation harness of \cref{sec:exp-distill}. Finally, we deploy the same closed-loop stack on the physical wheel-legged robot for every sphere count $k{=}1,\cdots,5$.

\end{enumerate}

% ============================================================================
\section{Related Work}
\label{sec:related}
% ============================================================================
\paragraph{Legged loco-manipulation.}
Rapid motor adaptation~\citep{kumar2021rma} and massively parallel reinforcement learning~\citep{rudin2022learning} have enabled increasingly agile locomotion on legged platforms. Building on these advances, DribbleBot~\citep{ji2023dribblebot} and subsequent work~\citep{he2024visual,liu2024visual} extended quadrupeds to dynamic single-object loco-manipulation, while LocoTouch~\citep{lin2025locotouch} introduced back-mounted tactile sensing for cargo transport. However, these studies focus primarily on single-object settings. To our knowledge, unconstrained transport of multiple freely moving spheres has not been studied, nor has per-frame permutation symmetry been identified as a representational challenge in scaling loco-manipulation to multiple objects.
\paragraph{Permutation-invariant networks.}
Deep Sets~\citep{zaheer2017deepsets} established the canonical form $\rho(\sum_i\phi(x_i))$; Set Transformer~\citep{lee2019settransformer}, PointNet~\citep{qi2017pointnet}, and Maron~\etal~\citep{maron2020sets} extended set processing to attention, point clouds, and higher-order tensors. These methods target single-frame sets, whereas our multi-frame setting permits independent permutations at each history frame and therefore requires the larger product symmetry $(N!)^{H-1}$.
\paragraph{Equivariant RL and multi-object symmetry.}
MDP homomorphic networks~\citep{vanderpol2020mdp,vanderpol2021multiagent}, SE(3)-equivariant manipulation~\citep{zhu2026equact,hoang2025geometry}, and multi-agent symmetry methods~\citep{mcclellan2024boosting} exploit geometric or agent-level structure. None addresses the temporal product symmetry produced by independent slot reassignment of identical objects across history frames.
\paragraph{Tactile sensing and distillation.}
Tactile sensing supports contact-rich manipulation and transport~\citep{dahiya2010tactile,luo2017robotic,bauza2020tactile,lloyd2023tactile,lin2025locotouch}, while DAgger~\citep{ross2011dagger} distills privileged policies into deployment-time inputs. Object-indexed teacher states do not map naturally to tactile measurements without persistent identities; we use their Boolean contact-footprint union, which is $\Gframe$-invariant and needs no explicit correspondence.

% ============================================================================
\section{Problem Formulation}
\label{sec:problem}
% ============================================================================

\paragraph{Multi-sphere transport MDP.} A self-designed robot with four independently steered wheels transports $k \in \{1,\cdots,N\}$ identical free-rolling steel spheres (radius \SI{2}{cm}) on a flat back plate. The maximum sphere count is $N=5$; the robot has 12 leg joints and 4 hub motors controlled at \SI{50}{Hz}. A $16{\times}16$ binary tactile array ($256$ cells) on the back plate provides contact feedback. Episodes terminate when a sphere leaves the plate or the robot falls. The sphere count increases through a multi-criterion curriculum based on episode length, support margin, dangerous-state fraction, edge overflow, and velocity-tracking error.

\paragraph{Multi-ball observation.} At policy time $t$, the agent observes proprioception $p_t\in\RR^{d_p}$ (history length 6; \cref{sec:obs}) and a multi-ball history tensor $X_t\in\RR^{H\times N\times d}$ with $H{=}4$ object frames and $d{=}14$ features (position, velocity, orientation, angular velocity, activity). Inactive slots are zero-masked. Since the spheres are identical, the physical equivalence class is $[X]_{\Gframe}=\{g\cdot X:g\in\Gframe\}$, and the value function should be constant on each orbit. We define the \emph{curriculum-reachable ball count} at budget $B$:
\begin{equation}
  K_{\max}(E; B) = \max\!\bigl\{k : \mathrm{Succ}_k(E; B) \ge \eta\bigr\},
\end{equation}
where $\mathrm{Succ}_k$ is the success rate on $k$-ball evaluation episodes and $\eta$ is the advancement threshold.

\paragraph{Evaluation metrics.} We report no-drop and strict success over 100 episodes per sphere count. Teachers use 1000 steps (\SI{20}{s}); the student uses 500 steps (\SI{10}{s}), matching distillation. No-drop requires every sphere to remain on the plate for the full horizon and is the primary metric. Strict additionally forbids near-edge events and therefore reflects both transport margin and curriculum exposure; 95\% Wilson intervals are approximately $\pm5$--$8$\,pp near 0 or 1 and about $\pm10$\,pp near 50\%.
\paragraph{Physical evaluation.} The same deployment stack runs closed-loop on hardware for $k{=}1,\cdots,5$, with synchronized video and binary tactile maps. Because matched repetitions are difficult on hardware, these experiments assess feasibility and sensing qualitatively; all quantitative success comparisons are performed in simulation.

% ============================================================================
\section{Method}
\label{sec:method}
% ============================================================================

\subsection{Per-Frame Deep Sets (\PFDS)}
\label{sec:pfds}

\PFDS separates \emph{within-frame} set aggregation from \emph{cross-frame} temporal fusion. Specifically, it extends standard Deep Sets~\citep{zaheer2017deepsets} as
\begin{align}
\label{eq:pfds}
h_t &= \Pool\!\bigl(\{\phi(x_{t,i})\}_{i \in \mathcal{A}_t}\bigr),
\quad t = 1,\cdots,H, \\
z   &= \rho\!\bigl([h_1;\,h_2;\,\cdots;\,h_H]\bigr),
\end{align}
where $\phi : \RR^d \to \RR^e$ is a shared per-object MLP, $\Pool$ is a continuous permutation-invariant aggregator over finite multisets (\cref{def:perm-inv}), and $\rho$ is an MLP that performs the temporal readout. We use masked sum pooling throughout, although the same construction also supports mean or max pooling.

A slot permutation $\pi_t$ acts jointly on the observations and activity mask, mapping $(\mathcal{A}_t,x_{t,\cdot})$ to $(\pi_t(\mathcal{A}_t),x_{t,\pi_t^{-1}(\cdot)})$. Thus $h_t$ depends only on the multiset $\{x_{t,i}\}_{i\in\mathcal{A}_t}$ and is unaffected by the ordering of active slots within frame $t$.

\begin{proposition}[$\Gframe$-invariance of \PFDS]
\label{prop:gframe-inv}
$f_{\mathrm{PF}}$ is $\Gframe$-invariant in the sense of \cref{def:gframe-inv}: for every $g \in \Gframe$ and every $X$,
$f_{\mathrm{PF}}(g \cdot X) = f_{\mathrm{PF}}(X)$.
\end{proposition} 
\textit{(Proof in \cref{sec:proofs}.)}

\begin{figure}[t] \centering \includegraphics[width=0.92\linewidth]{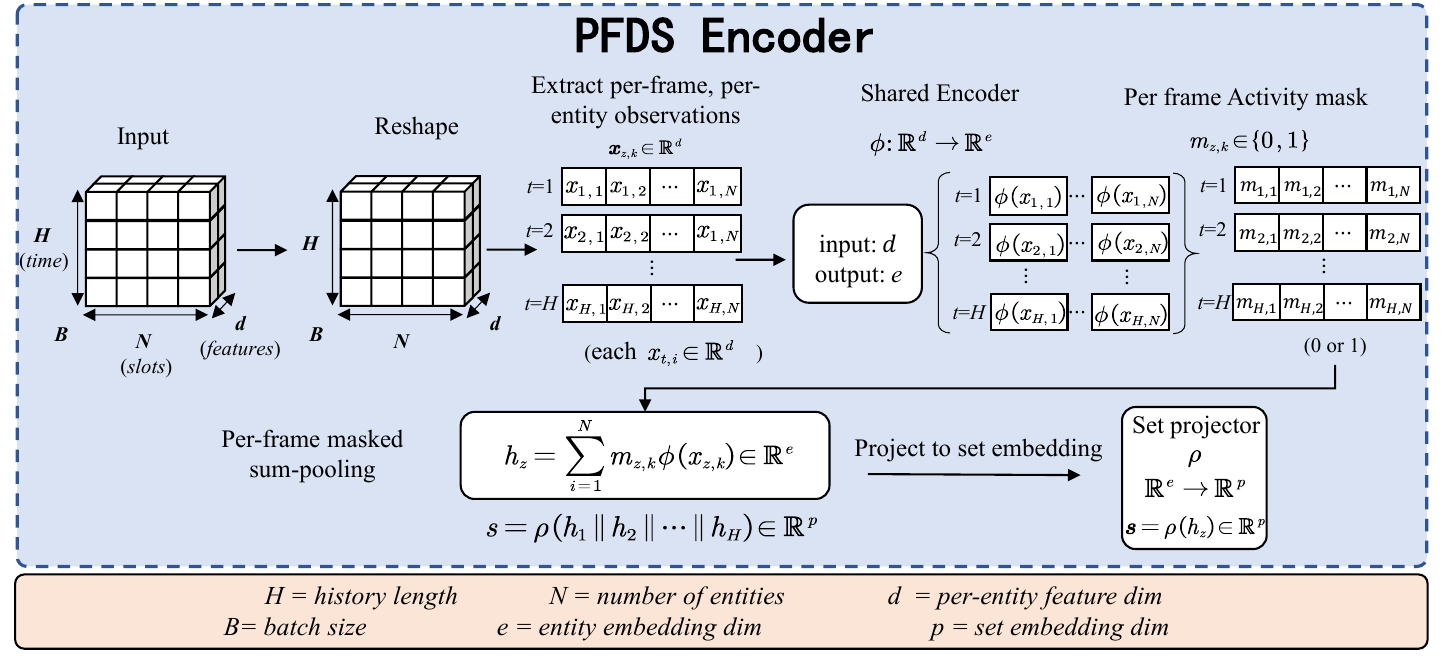} \caption{\PFDS encoder architecture. Per-object observations $x_{t,i}$ are independently pooled within each history frame to obtain frame-level embeddings $h_t$, which are then concatenated and passed to the temporal readout MLP $\rho$. This frame-wise pooling makes the encoder invariant to independent permutations of object slots across frames, i.e., $\Gframe$-invariant by \cref{prop:gframe-inv}.}
\label{fig:compare}
\end{figure}

\paragraph{Contrast with \DSHC.}
Applying standard Deep Sets~\citep{zaheer2017deepsets} to multi-frame histories first concatenates each slot through time, $y_i = [x_{1,i};\cdots;x_{H,i}] \in \RR^{Hd}$, and then aggregates:
\begin{equation}
\label{eq:hcds}
f_{\mathrm{HC}}(X) = \rho\!\Bigl(\sum_{i=1}^{N} \phi_Z(y_i)\Bigr),
\quad \phi_Z : \RR^{Hd} \to \RR^e.
\end{equation}

We call this construction \DSHC (Deep Sets with History Concatenation). Since each token ties one slot index across time, it is invariant only to a common permutation of all frames: $f_{\mathrm{HC}}$ is $\Gdiag$- but not generally $\Gframe$-invariant. \PFDS pools within each frame before temporal readout and therefore removes this cross-frame slot dependence.

\begin{proposition}[$\Gdiag$ but not $\Gframe$ invariance of \DSHC]
\label{prop:gdiag}
$f_{\mathrm{HC}}$ is $\Gdiag$-invariant but not generally $\Gframe$-invariant: for any $H\ge2$ and $N\ge2$, there exist a continuous feature map (e.g., $\phi_Z(y)=\mathrm{vec}(yy^{\!\top})$), readout $\rho$, input $X$, and $g\in\Gframe\setminus\Gdiag$ such that $f_{\mathrm{HC}}(g\cdot X)\ne f_{\mathrm{HC}}(X)$.
\end{proposition}

\textit{(Proof in \cref{sec:proofs}.)}

% \smallskip

The remaining symmetry gap is $|\Gframe|/|\Gdiag|=(N!)^{H-1}=(5!)^3\approx1.7{\times}10^6$ for $N{=}5,H{=}4$.

\begin{theorem}[Universal approximation of \textup{PFDS} for $\Gframe$-invariant functions]
\label{thm:approx}
Fix the per-frame cardinalities $m_t=|\mathcal{A}_t|$ and a compact box $K\subset\RR^d$, and let $\Omega=\prod_tK^{m_t}$ with $G=\prod_t\mathfrak{S}_{m_t}$ acting independently within each frame. On the fixed-cardinality active-slot slice, if $\phi$ and $\rho$ are sufficiently wide continuous MLPs and $\Pool$ is sum-pooling, then \textup{\PFDS} is dense (under the sup-norm) in the continuous $G$-invariant functions $f:\Omega\to\RR^a$, equivalently in continuous functions on $\Omega/G$.
\end{theorem}

\textit{(Proof in \cref{sec:proofs}; the result follows by applying Deep Sets universality independently to each frame, followed by a universal MLP readout.)}

% \smallskip

\subsection{Tactile Student Policy (TactSet)}
\label{sec:tactset}

The \PFDS teacher uses privileged ball-state observations $X_t$, unavailable on the physical robot. At deployment, the student instead receives a binary tactile map from the back-mounted array. \TactSet imitates the teacher using $\tau_t \in \{0,1\}^{R\times C}$, the Boolean union of all active-sphere contact footprints. Let $(\bar x_{t,i},\bar y_{t,i},\bar z_{t,i})$ denote the center of sphere $i$ at time $t$ in the back-plate frame, and $(\bar x_s,\bar y_s)$ the planar center of tactile cell $s$:
\begin{equation}
\label{eq:tac}
\tau_{t,s} = \bigvee_{i \in \mathcal{A}_t}
\mathbf{1}\!\bigl(|\bar x_{t,i} - \bar x_s| < r + \delta_x,\;
|\bar y_{t,i} - \bar y_s| < r + \delta_y,\;
|\bar z_{t,i} - (z_{\mathrm{plate}} + r)| < \epsilon_z\bigr),
\end{equation}
where $z_{\mathrm{plate}}$ denotes the plate-surface height, and the $z$-condition gates sphere--plate contact. In simulation we rasterize the map on a $16{\times}16$ grid; the physical platform uses a $16{\times}16$ M1616M sensor, both expose a binary-map interface.

Because Boolean union is symmetric, $\tau_t$ depends only on contact footprints, not sphere identities or slot assignments. It is therefore naturally $\Gframe$-invariant and provides a permutation-consistent student representation.

The student receives $s_t=[\tau_t;\tau_{t-\delta_t}]$, with $\delta_t$ sampled uniformly from $\{1,2\}$ frames per training episode. A compact CNN (embedding 64) and GRU (hidden size 512) encode these maps; the result is concatenated with proprioception and passed to the policy head (hyperparameters in \cref{tab:supp_distill}). The student matches teacher actions by minimizing
\begin{equation}
\label{eq:distill}
\min_\theta\;
\EE_{(p_t, X_t, s_t) \sim \mathcal{D}}\!\left[
\bigl\|\pi_S^\theta(p_t, s_t) - \pi_T(p_t, X_t)\bigr\|_2^2
\right],
\end{equation}
using DAgger-style aggregation: behavior-cloning initialization followed by seven online rollout rounds labeled by the \PFDS teacher. The same delayed two-channel representation is deployed; the physical sensor publishes binary $16{\times}16$ maps at \SI{10}{Hz} with a \SI{5}{g} threshold.

% ============================================================================
\section{Experiments}
\label{sec:exp}

\subsection{Setup}

All teacher policies use Isaac Lab~\citep{mittal2025isaaclab} and PPO~\citep{schulman2017ppo} with 4096 environments for 30\,000 iterations. Rewards, curriculum, and hyperparameters are shared; only the set encoder varies. Tables use \texttt{perm. True/False}, while figure legends use \texttt{perm=T/F}; full settings are in \cref{tab:supp_hyperparams}.

\paragraph{Physical platform.}
We use a custom wheel-legged robot with four independently steered wheels, 12 leg joints, and 4 wheel motors. The policy runs at \SI{50}{Hz} and the low-level controller at \SI{500}{Hz}. A \SI{15.1}{cm}$\times$\SI{15.1}{cm} back plate is covered by silicone and compliant foam above a $16{\times}16$ M1616M binary pressure array sampled at \SI{10}{Hz} with a \SI{5}{g} threshold. The policy runs closed-loop for $k{=}1,\cdots,5$; synchronized video and tactile data are recorded for $110$--$119$\,s per condition (alignment residual below \SI{0.5}{s}). These trials are qualitative only.

\paragraph{Encoder baselines.}
We compare eight architectures spanning three symmetry levels (\cref{tab:encoders}), with matched parameter budgets.
\textit{$G_0$ (none):} \FlatMLP flattens $X$; \BranchMLP applies a shared MLP to each slot history and concatenates the results. \textit{$\Gdiag$:} \DSHC~\citep{zaheer2017deepsets} aggregates history-concatenated tokens, while Set Transformer~\citep{lee2019settransformer} replaces pooling with ISAB/PMA attention. \textit{$\Gframe$ (ours):} \PFDS pools independently within each frame before temporal readout (\cref{prop:gframe-inv}); Set Transformer (PF) is its higher-capacity attention analogue. We do not distill from it because it is the most expensive $\Gframe$ encoder. \PFDSOFF disables augmentation while retaining the architecture.

\subsection{Main Comparison: Architecture Determines Curriculum Reachability}
\label{sec:exp-main}

\begin{table}[t]
\centering
\small
\caption{Teacher encoder comparison under the same PPO training budget (30\,000 iterations). $K_{\max}$ denotes the highest curriculum level reached. `Adv.\ to $k$'' reports the number of training iterations (in thousands) required to first reach that level; `---'' indicates that the level was not reached within the budget. Inv.\ denotes the invariance group. Strict/no-drop report 5-ball success rates over 100 trials (95\% Wilson CI $\approx\pm5$--$8$ pp), with no-drop treated as the primary metric (\cref{sec:problem}). The \DSHCOFF row uses the originally trained seed (42); additional seeds are reported in \cref{app:postreview}. $\dagger$ denotes the only $\Gframe$ encoder that reaches $k{=}5$ without slot-permutation augmentation.}
 \label{tab:encoders} 
  \setlength{\tabcolsep}{3pt} \begin{tabular}{lccccccc} \toprule \textbf{Encoder} & \textbf{Inv.} & \textbf{perm.} & \textbf{$K_{\max}$} & \textbf{Adv.\ to 3} & \textbf{Adv.\ to 5} & \textbf{strict} & \textbf{no-drop} \\ \midrule \FlatMLP                   & $G_0$      & True & $\le 1$ & ---            & ---            & $0\%$  & $0\%$ \\ \BranchMLP                 & $G_0$      & True & $\le 2$ & ---            & ---            & $0\%$  & $2\%$ \\ \DSHCOFF                   & $\Gdiag$   & False & $\le 2$ & ---            & ---            & $0\%$  & $0\%$ \\ \DSHC                      & $\Gdiag$   & True & $5$     & ${\sim}7.4$k   & ${\sim}10.9$k  & $68\%$ & $98\%$ \\ Set Transformer            & $\Gdiag$   & True & $5$     & ${\sim}3.6$k   & ${\sim}7.7$k   & $10\%$ & $94\%$ \\ Set Transformer (PF, ours) & $\Gframe$  & True & $5$     & ${\sim}7.5$k   & ${\sim}16.4$k  & $54\%$ & $100\%$ \\\PFDSOFF (ours)$^\dagger$ & $\Gframe$  & False & $5$     & ${\sim}16.9$k  & ${\sim}27.9$k  & $41\%$ & $98\%$ \\ \textbf{\PFDS (ours)}      & $\Gframe$  & True & $5$     & ${\sim}3.9$k   & ${\sim}6.3$k   & $41\%$ & $100\%$ \\ \bottomrule
  \end{tabular}
\end{table}

\Cref{fig:training_curves} shows curriculum progression and \cref{fig:success_rates} success rates. $G_0$ encoders stall at or below $k{=}2$; $\Gdiag$ encoders reach $k{=}5$ with augmentation but the unaugmented \DSHCOFF run stalls at $k{=}2$; $\Gframe$ encoders reach $k{=}5$ in both regimes. Additional \DSHCOFF seeds appear in \cref{app:postreview}; without explicit $\Gframe$ invariance, progression is less reliable within budget.

\paragraph{Primary metric (no-drop).}
\PFDS and Set Transformer~(PF) achieve 100\% no-drop at $k{=}5$ (\cref{tab:encoders,fig:success_rates}); \PFDS reaches $k{=}5$ fastest (6.3k iterations, versus 10.9k for \DSHC and 7.7k for Set Transformer). The 2\,pp gap over \DSHC is within 100-trial sampling uncertainty, but \PFDSOFF still reaches $k{=}5$ with 98\% no-drop without augmentation, whereas \DSHCOFF stalls at $k{=}2$. Thus $\Gframe$ architectures did not require augmentation in our runs.

\begin{figure}[t]
\centering
\includegraphics[width=\linewidth]{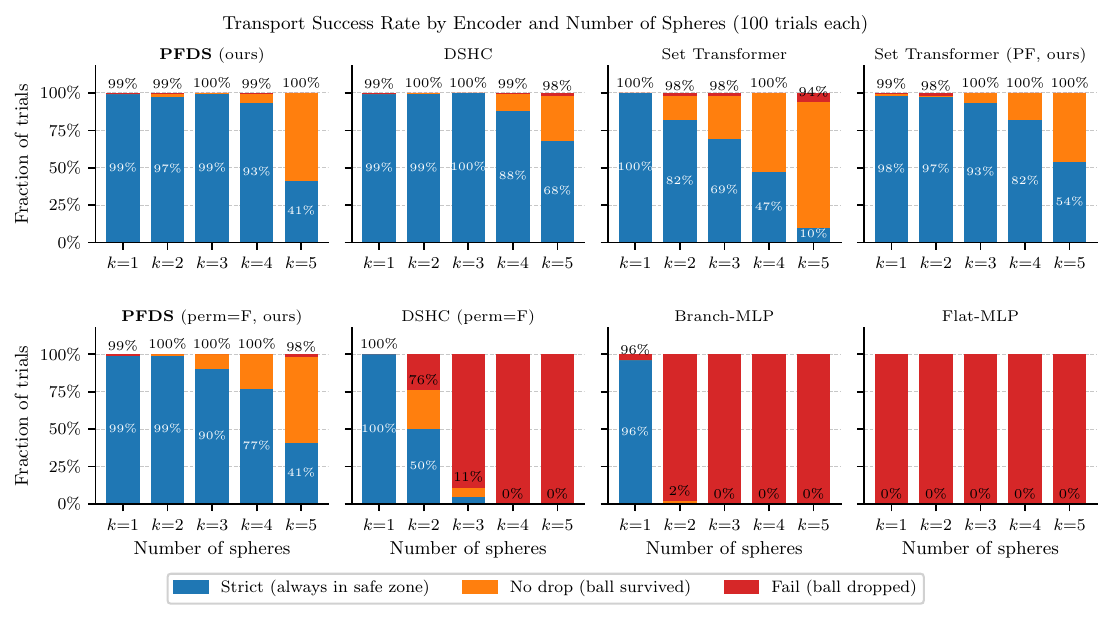}

\caption{Transport success over 100 trials for five ball counts. Blue: strict; orange: no-drop; red: failure. At $k{=}5$, $\Gframe$ encoders achieve 100\% no-drop, while \DSHC (perm.\ True) achieves 98\%. $G_0$ encoders fail beyond $k{=}1$ and \DSHCOFF stalls at $k{=}2$ (additional seeds in \cref{app:postreview}).}
\label{fig:success_rates}
\end{figure}

\begin{figure}[t]
\centering
\begin{minipage}[t]{0.58\linewidth}
\centering
\includegraphics[width=\linewidth]{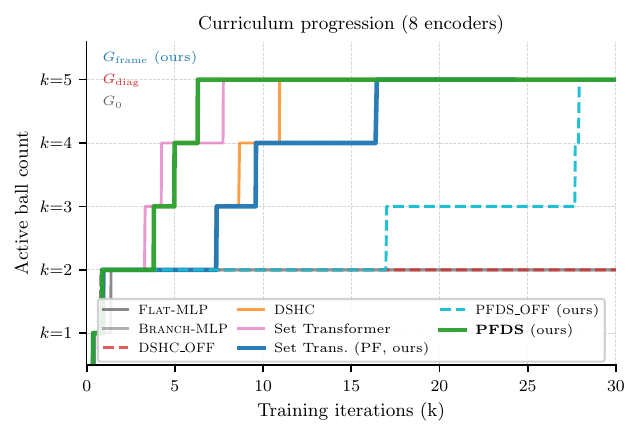}
\end{minipage}
\hfill
\begin{minipage}[t]{0.39\linewidth}
\centering
\includegraphics[width=\linewidth]{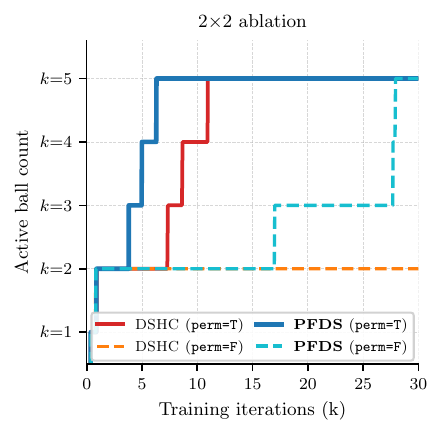}
\end{minipage}

\caption{Training dynamics. \textbf{Left}: curriculum progression for eight encoders. $\Gframe$ models reach $k{=}5$, $\Gdiag$ models do so only with augmentation, and $G_0$ models stall at $k{\le}2$. \textbf{Right}: $2{\times}2$ ablation; \PFDS reaches $k{=}5$ with or without augmentation, whereas unaugmented \DSHC stalls at $k{=}2$.}
\label{fig:training_curves}
\label{fig:ablation}
\end{figure}

\subsection{\texorpdfstring{$2 \times 2$}{2 x 2} Key Ablation: Two Paths to \texorpdfstring{$\Gframe$}{G-frame} Robustness}

\label{sec:exp-2x2}

\begin{table}[t]
\centering
\small
\caption{$2 \times 2$ ablation over encoder architecture and per-step orbit augmentation. `Adv.\ to 3'' denotes the number of training iterations required for the curriculum to first reach $k \ge 3$; `---'' indicates that this level is not reached within 30\,000 iterations.}
\label{tab:perm_ablation}
\begin{tabular}{lcc|cc}
\toprule
& \multicolumn{2}{c|}{\textbf{permute=True}}
& \multicolumn{2}{c}{\textbf{permute=False}} \\
\textbf{Encoder} & $K_{\max}$ & Adv.\ to 3 & $K_{\max}$ & Adv.\ to 3 \\
\midrule
\DSHC ($\Gdiag$) & $5$ & ${\sim}7.4$k & $\le 2$ & --- \\
\PFDS ($\Gframe$) & $5$ & ${\sim}3.9$k & $5$ & ${\sim}16.9$k \\
\bottomrule
\end{tabular}
\end{table}

\Cref{tab:perm_ablation,fig:ablation} show the central pattern: \DSHC reaches $k{=}5$ with per-step augmentation ($\pi_t\!\sim\!\mathrm{Uniform}(\SN)$) but stalls at $k{=}2$ without it, whereas \PFDS reaches $k{=}5$ in both settings (more slowly without augmentation). The slot index can otherwise become a curriculum-stage shortcut; the mechanism is analyzed in \cref{sec:discussion}.
\subsection{Additional Robustness, Diagnostics, and Distillation}

\label{sec:exp-distill}

\paragraph{Additional evidence.} Full tables and protocols are in \cref{app:postreview}.

The \DSHCOFF failure is seed-dependent: two additional seeds eventually reach $k{=}5$ (at approximately 15k and 22k iterations), but remain at $k{=}2$ for most of training and reach only $22$--$23\%$ strict success (no-drop $87$--$100\%$), versus $39$--$41\%$ strict / $98$--$100\%$ no-drop for \PFDS. Thus the augmentation-free curriculum is slow and low-yield, not impossible.

\emph{Slot shortcut and invariance.} Test-time relabeling reduces \DSHCOFF strict success at $k{=}2$ from $91\%$ to $53\%$ ($-38$\,pp); \PFDS and \DSHC remain within sampling noise. At the action level, \PFDS changes by only $3.6$--$6.4{\times}10^{-4}$ of the action standard deviation under $\Gframe$ permutations, whereas augmentation-trained \DSHC deviates by $9.5$--$20\%$ ($\sim265$--$315\times$ larger), especially in low-margin states.

\emph{Capacity versus curriculum.} Directly trained at $k{=}5$, \DSHCOFF and \PFDSOFF achieve $98\%$ and $100\%$ no-drop, respectively (strict $73\%/70\%$ and $21\%/32\%$ under perm.\ True/False), showing that the plateau is a curriculum effect rather than a capacity limit. These fixed-$k$ specialists score $0\%$ at every $k<5$, whereas curriculum training yields one policy for all counts.

\emph{\textup{\TactSet} robustness and distillation.} In the matched robustness harness, \TactSet scores $85\%$ no-drop at $k{=}5$. Taxel dropout of 1--10\% costs at most 1\,pp; 20\% dropout and spurious contacts degrade more gradually. Calibration shifts yield 86\%/61\%, and latency gives 71\% at $+1$ frame and 48\% at $+3$. Behavior cloning reaches 4\%, one DAgger round 72\%, and the shipped round-7 student 85\%, indicating that most improvement is on-policy while the remaining gap reflects information discarded by the binary map.

\emph{Transport quality.} At $k{=}5$, \PFDS retains all spheres in 99\% of episodes. All 59 strict failures are transient near-edge excursions with no drops (mean dwell 0.1\%), so strict success measures containment margin rather than basic transport. Static packing is not limiting at $r{=}\SI{2}{cm}$; stronger margin objectives trade off against velocity tracking, whose error rises from \SI{0.059}{m/s} at $k{=}4$ to \SI{0.068}{m/s} at $k{=}5$.

\begin{figure}[h]
\centering
\includegraphics[width=0.98\linewidth]{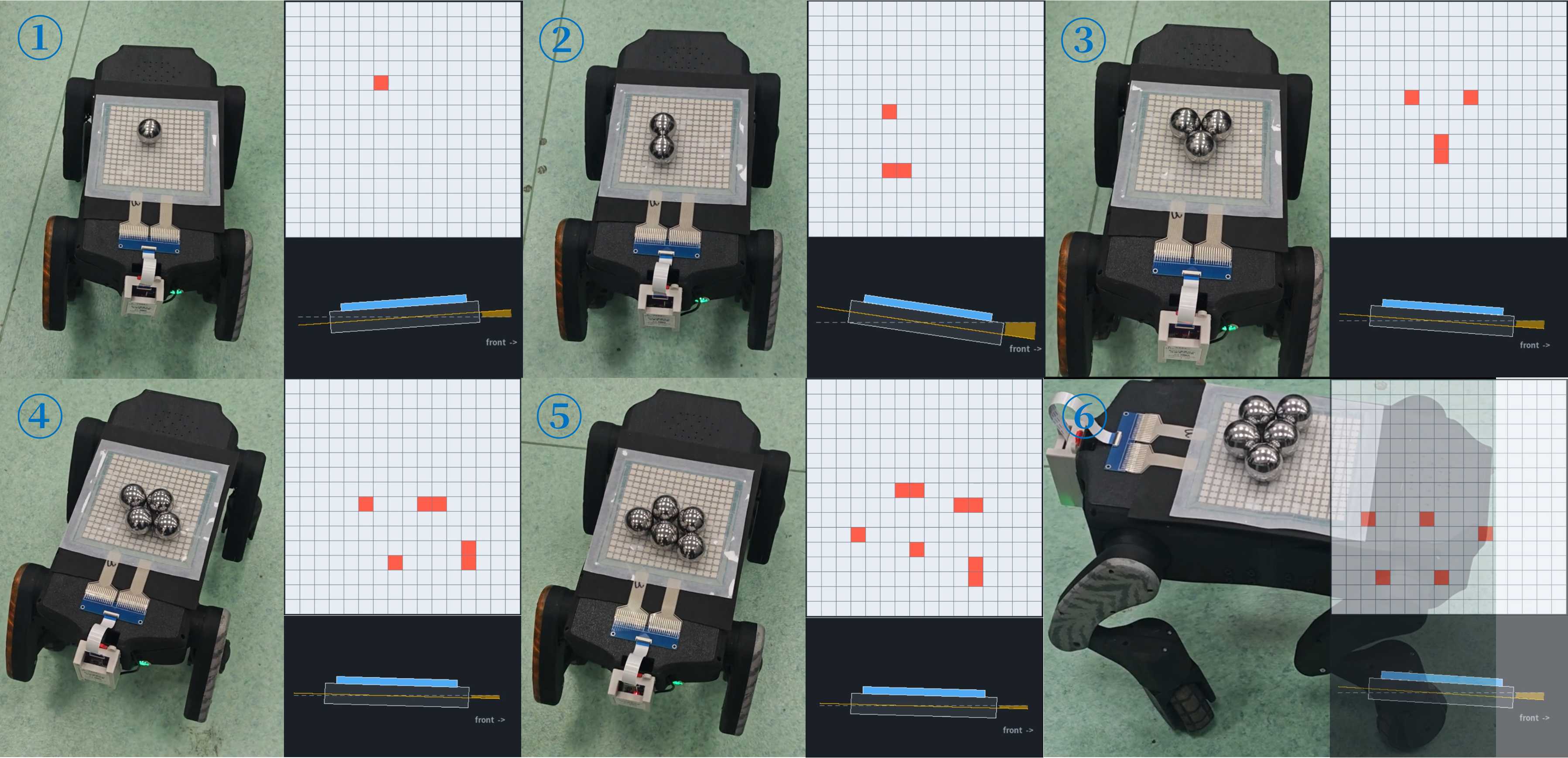}

\caption{Physical validation on the wheel-legged robot. Panels 1--5 pair representative frames for $k{=}1,\ldots,5$ spheres with measured $16{\times}16$ binary contact maps and body-pitch indicators.}

\label{fig:hw}
\end{figure}

The same diagnostics explain why augmentation is helpful but not sufficient: it exposes a $\Gdiag$ encoder to more relabelings, yet does not enforce invariance. The fixed-$k$ specialists further show that the representation can solve the physical task when curriculum interference is removed, but cannot replace a curriculum policy that generalizes over all sphere counts. For \TactSet, the residual teacher--student gap is consistent with the binary union map discarding velocities and object-wise attribution; latency and calibration are therefore deployment-critical sensitivities.

\subsection{Physical Validation}

\label{sec:hardware}

\begin{wrapfigure}{r}{0.40\linewidth}
  \vspace{-10pt}
  \centering
  \includegraphics[width=\linewidth]{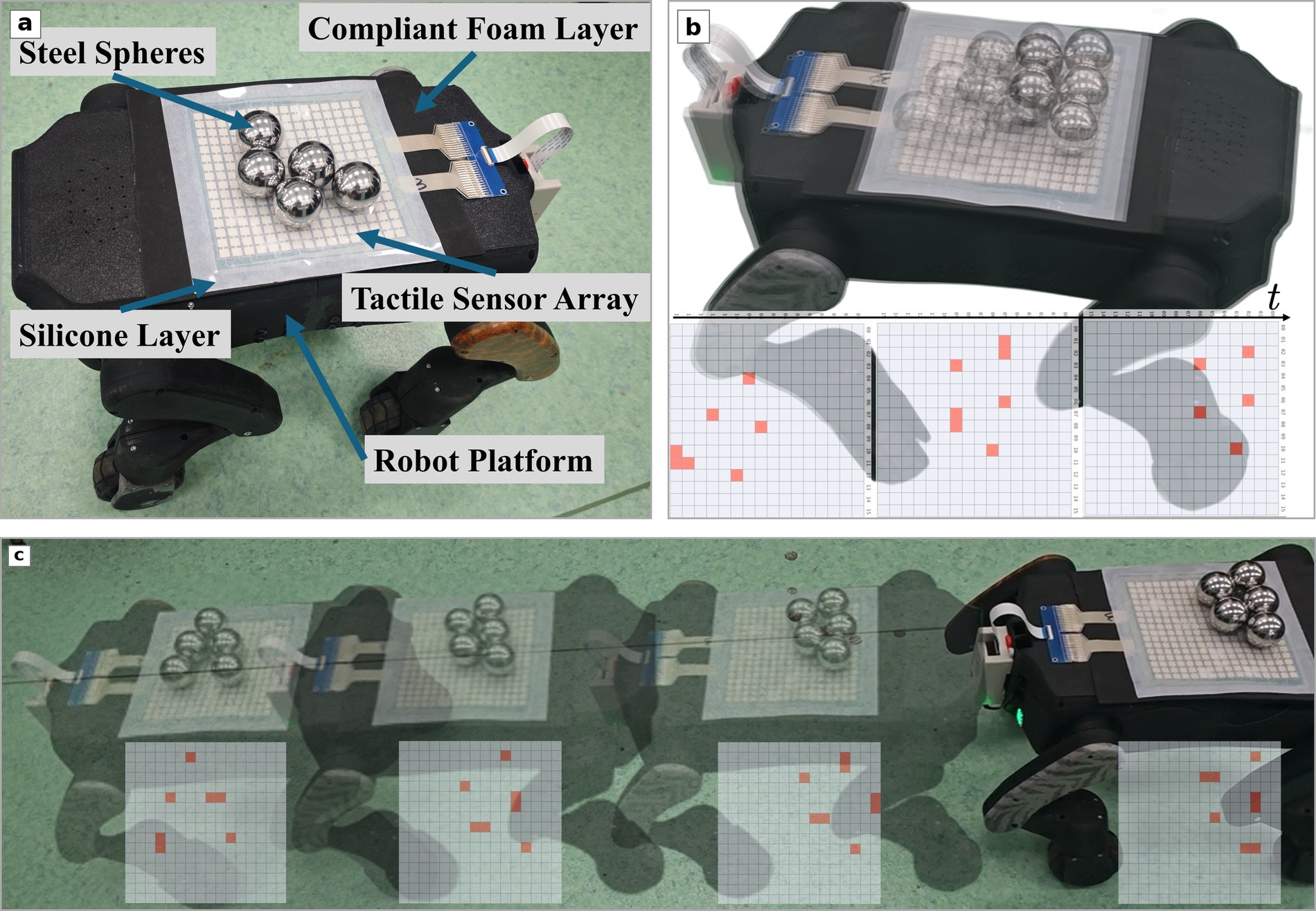}
  \caption{Physical platform and tactile sensing. The $16{\times}16$ array produces identity-free binary maps in qualitative trials.}
  \label{fig:hardware-stack}
  \vspace{-10pt}
\end{wrapfigure}

We run the complete system closed-loop on hardware for every $k{=}1,\cdots,5$, recording synchronized video and measured $16{\times}16$ binary maps. In \cref{fig:hw}, the selected map has one 8-connected component per sphere but no persistent identity; object-indexed state recovery would require frame-to-frame association.

The full-width simulation rollout in \cref{fig:sim-rollouts} illustrates the
open-plate task across all five sphere counts, while the embedded view in
\cref{fig:hardware-stack} exposes the deployment sensor stack and binary maps.
Together, these views connect the quantitative simulation protocol to the
qualitative hardware interface; neither is an additional hardware success-rate
estimate.

% ============================================================================
\section{Conclusion}
\label{sec:conclusion}
% ============================================================================

We introduced Per-Frame Deep Sets (\PFDS), an intrinsically $\Gframe$-invariant encoder for multi-sphere transport. The eight-encoder comparison and $2{\times}2$ ablation show the benefit of frame-wise invariance without augmentation; additional seeds clarify that unaugmented \DSHC is slow and seed-dependent rather than impossible (\cref{app:dshc-seeds}). \PFDS reaches $k_{\max}{=}5$ across five seeds, \TactSet distills it into a tactile student, and the same stack runs on hardware at every $k{=}1,\cdots,5$.

\paragraph{Limitations and future work.}
Quantitative results are simulation-based, while hardware trials are qualitative. Future work will calibrate tactile responses, run repeated randomized trials, and study temporal association under varied terrain and disturbances.

\acknowledgments{This work was supported by the National Natural Science Foundation of China (Grant No. 52675021) and the Fundamental and Interdisciplinary Disciplines Breakthrough Plan of the Ministry of Education of China (JYB2025XDXM208).}

% ── references ──────────────────────────────────────────────────────────────
\bibliography{references}

% ============================================================================
% Appendix
% ============================================================================
\appendix

\section{Proof Details}
\label{sec:proofs}

\subsection{Preliminaries: Permutation and \texorpdfstring{$\Gframe$}{G-frame}-Invariance}
\label{sec:invariance-defs}

For completeness we recall the two notions of invariance used throughout the paper.

\begin{definition}[Permutation invariance of a multiset aggregator]
\label{def:perm-inv}
Let $\mathcal{M}_{\le N}(\RR^e)$ denote the set of finite multisets in $\RR^e$ of cardinality at most $N$. A continuous map $\Pool:\mathcal{M}_{\le N}(\RR^e)\to\RR^e$ is \emph{permutation-invariant} if its value depends only on its argument as a multiset (with multiplicity), i.e.\ for any finite indexed family $(z_i)_{i\in\mathcal{A}}$ with $|\mathcal{A}|\le N$ and any bijection $\sigma:\mathcal{A}\to\mathcal{A}$,
\[
  \Pool\!\bigl(\{z_{\sigma(i)}\}_{i\in\mathcal{A}}\bigr) = \Pool\!\bigl(\{z_i\}_{i\in\mathcal{A}}\bigr).
\]
Equivalently, for each cardinality $m \le N$, the restriction of $\Pool$ to multisets of size $m$ corresponds to a symmetric continuous function $(\RR^e)^m \to \RR^e$. Standard examples include masked sum, mean, and max.
\end{definition}

\paragraph{Robot interpretation.}
The indices in a simulator tensor are bookkeeping labels, not physical
identities.  If the same two balls are stored in slots 1 and 2 in one
observation and in slots 2 and 1 in the next, a robot should receive the same
command whenever the physical configuration is the same.  A masked sum (or
mean/max) implements this idea by applying the same feature map to every
active ball and then combining the values without reading the slot index.  The
activity mask is important: inactive slots contribute nothing, so changing the
number of balls does not introduce an artificial ``zero ball'' into the pooled
representation.

\begin{definition}[$\Gframe$-invariance of a function on history tensors]
\label{def:gframe-inv}
With $\Gframe = (\SN)^H$ acting on $X\in\RR^{H\times N\times d}$ by $(g\cdot X)_{t,i} = x_{t,\pi_t^{-1}(i)}$ for $g=(\pi_1,\cdots,\pi_H)$, a function $f:\RR^{H\times N\times d}\to\RR^a$ is \emph{$\Gframe$-invariant} if
\[
  f(g\cdot X) = f(X) \qquad \text{for all } X \text{ and all } g\in\Gframe.
\]
A history tensor $X$ and its $\Gframe$-orbit $[X]_{\Gframe} = \{g\cdot X : g\in\Gframe\}$ are identified under any $\Gframe$-invariant $f$, so $f$ factors through the quotient $\RR^{H\times N\times d}/\Gframe$.
\end{definition}

\paragraph{Why this matters for transport.}
One element of the orbit $[X]_{\Gframe}$ is one possible assignment of the
same visible balls to slots at each time step.  These assignments are
physically indistinguishable when the sensor reports ball attributes but no
persistent identity.  The quotient therefore has a direct operational
meaning: a policy sees only information justified by the observation and
cannot infer which ball in the current frame was called ``ball 1'' in a
previous frame.  This is stronger than ordinary data augmentation, which only
encourages similar outputs on relabeled examples.

\subsection{Proof of PFDS \texorpdfstring{$\Gframe$}{G-frame}-Invariance}

\begin{proof}
Per (\ref{eq:gframe}), under $g = (\pi_1,\cdots,\pi_H) \in \Gframe$ the new slot $i$ at frame $t$ holds the old value at slot $\pi_t^{-1}(i)$; equivalently, the old slot $j$ moves to $\pi_t(j)$, so the active set transforms as $\mathcal{A}_t \mapsto \pi_t(\mathcal{A}_t)$. Therefore
\[
  h_t(g \cdot X)
  = \Pool\!\bigl(\{\phi(x_{t,\pi_t^{-1}(i)})\}_{i\in\pi_t(\mathcal{A}_t)}\bigr)
  = \Pool\!\bigl(\{\phi(x_{t,j})\}_{j\in\mathcal{A}_t}\bigr)
  = h_t(X),
\]
where the second equality reindexes the multiset by $j = \pi_t^{-1}(i)$, which traverses $\mathcal{A}_t$ exactly once. The two multisets $\{\phi(x_{t,j})\}_{j \in \mathcal{A}_t}$ are identical, so $\Pool$, being a function on multisets, returns the same value. Hence $z(g \cdot X) = \rho([h_1; \cdots; h_H]) = z(X)$.
\end{proof}

\paragraph{Operational takeaway.}
The proof uses only two facts: pooling is a function of a multiset, and the
active mask is permuted together with the values.  Each frame is therefore
reduced to a label-free summary $h_t$ before temporal fusion.  The temporal
module still sees how the aggregate support configuration changes over time,
but it cannot exploit a cross-frame identity link that the sensor does not
provide.

\subsection{Proof of DSHC \texorpdfstring{$\Gdiag$}{G-diag}- but not \texorpdfstring{$\Gframe$}{G-frame}-Invariance}

\emph{$\Gdiag$-invariance.} For $g = (\pi,\cdots,\pi) \in \Gdiag$, every frame is permuted by the same $\pi$, so each per-slot token transforms as $y_i(g\cdot X) = y_{\pi^{-1}(i)}(X)$. The set $\{y_i\}_{i=1}^N$ is therefore permuted as a multiset, and
\[
  \sum_{i=1}^{N} \phi_Z(y_i(g\cdot X)) = \sum_{i=1}^{N} \phi_Z(y_{\pi^{-1}(i)}(X)) = \sum_{i=1}^{N} \phi_Z(y_i(X)),
\]
so $f_{\mathrm{HC}}(g \cdot X) = f_{\mathrm{HC}}(X)$.

\emph{Failure of $\Gframe$-invariance ($H{=}2$, $N{=}2$ counterexample).} Let $X = \begin{pmatrix} a & b \\ c & d \end{pmatrix}$ with rows indexing time and columns indexing slots, and take $g = (\mathrm{id}, (12)) \in \Gframe \setminus \Gdiag$. Then $(g \cdot X)_{1,:} = (a, b)$ and $(g \cdot X)_{2,:} = (d, c)$, so the \DSHC tokens are $y_1^g = (a, d)$, $y_2^g = (b, c)$, whereas $y_1 = (a, c)$, $y_2 = (b, d)$. For the linear feature $\phi_Z = \mathrm{id}$, $\sum_i y_i^g = (a+b, c+d) = \sum_i y_i$: the linear feature is incidentally invariant. For the flattened outer-product feature $\phi_Z(y) = \mathrm{vec}(yy^{\!\top})$, a direct computation gives
\[
  \sum_i y_i y_i^{\!\top}
  = \begin{pmatrix} a^2 + b^2 & ac + bd \\ ac + bd & c^2 + d^2 \end{pmatrix},
  \quad
  \sum_i y_i^g (y_i^g)^{\!\top}
  = \begin{pmatrix} a^2 + b^2 & ad + bc \\ ad + bc & c^2 + d^2 \end{pmatrix},
\]
so the $(1,2)$-entry differs by $(ad + bc) - (ac + bd) = (a-b)(d-c)$, which is nonzero whenever $a \ne b$ and $c \ne d$. Choose $\rho$ to project onto the coordinate of $\mathrm{vec}(\cdot)$ holding the $(1,2)$-entry. Then $f_{\mathrm{HC}}(g \cdot X) - f_{\mathrm{HC}}(X) = (a-b)(d-c) \ne 0$, exhibiting the required $\phi_Z, \rho, X$, and $g \in \Gframe \setminus \Gdiag$. The construction trivially extends to any $H\ge 2$, $N\ge 2$ by padding additional frames and additional slots with constants. \qed

\paragraph{Interpretation for the robot.}
The diagonal group allows one global relabeling for the whole history, whereas
$\Gframe$ allows labels to change independently at every frame.  A policy that
is only $\Gdiag$-invariant can therefore compare slot 1 at time $t$ with slot 1
at time $t{+}1$ even when those slots contain different physical balls.  In
the counterexample, a linear sum happens to hide this mistake, but a nonlinear
feature exposes the changed cross-time pairing.  The example does not claim
that every trained DSHC policy fails; it shows that the architecture can
represent a frame-wise identity shortcut, so the desired invariance is not
guaranteed by construction.

\subsection{Proof Sketch: Universal Approximation of \PFDS}

Let $m_t$, $K$, $\Omega$, $G$, $\mathcal{X}$ be as in the theorem, and let $n_{\max} = \max_t m_t$. Since each frame has \emph{exactly} $m_t$ active slots, an orbit in $\Omega/G$ is identified with a tuple of multisets $(S_1,\cdots,S_H)$ with $|S_t| = m_t$ (no cardinality ambiguity arises across frames).

\emph{(i) Per-frame injection by symmetric power sums.}  For each fixed $m$ and any compact $K \subset \RR^d$, the map
\[
  \phi_m : \RR^d \to \RR^{M_m}, \qquad \phi_m(x) = \bigl(p_{\alpha}(x)\bigr)_{|\alpha|\le m},
\]
collecting all monomials of degree at most $m$ in the $d$ coordinates of $x$, makes $\Phi_m(S) = \sum_{x\in S}\phi_m(x)$ a continuous injection on multisets of \emph{exactly} cardinality $m$: the resulting multivariate power-sum symmetric functions determine the elementary multisymmetric polynomials of $S$ via the multisymmetric Newton identities, and these in turn determine $S$ uniquely. (We use that $|S|$ is known at the time of decoding, so the multiset $\{0,\cdots,0\}$ at one cardinality is distinguishable from the same elements at any other cardinality by context.) Taking $\phi$ to include the components $\phi_m$ for all $m \in \{m_1,\cdots,m_H\}$ (or, equivalently, all $m \le n_{\max}$) produces a single continuous $\phi:\RR^d\to\RR^M$ that is shared across frames, matching the \PFDS architecture in (\ref{eq:pfds}). Each per-frame embedding $\Phi_t(S_t) = \sum_{x \in S_t}\phi(x)$ is then continuous and injective on $K^{m_t}/\mathfrak{S}_{m_t}$.

\emph{(ii) Joint injection.} Concatenating per-frame embeddings,
\[
  v(X) = [\Phi_1(X_1);\cdots;\Phi_H(X_H)] \in \RR^{HM},
\]
gives a continuous map $\Omega \to \RR^{HM}$ that factors through $\mathcal{X}$ and is injective on $\mathcal{X}$, since two orbits agree on $v$ iff their per-frame multisets coincide for every $t$. Since $\Omega$ is compact, so is $\mathcal{X}$, and a continuous bijection from a compact space to a Hausdorff space (here $v: \mathcal{X} \to v(\mathcal{X}) \subset \RR^{HM}$) is a homeomorphism.

\emph{(iii) Readout approximation.} A continuous $G$-invariant function $f:\Omega \to \RR^a$ descends to a unique continuous $\bar f: \mathcal{X}\to\RR^a$, which by (ii) corresponds to a continuous $\tilde f = \bar f \circ v^{-1} : v(\mathcal{X}) \to \RR^a$. Because $v(\mathcal{X})$ is compact (hence closed) in the normal space $\RR^{HM}$, Tietze's extension theorem (applied coordinate-wise) extends $\tilde f$ to a continuous $\hat f : \RR^{HM}\to\RR^a$. Restricting $\hat f$ to a compact neighbourhood of $v(\mathcal{X})$ and applying the universal approximation theorem for multilayer perceptrons~\citep{hornik1989universal} yields a continuous $\rho : \RR^{HM} \to \RR^a$ such that $\sup_{X \in \Omega} \|\rho(v(X)) - f(X)\| < \varepsilon$ for any prescribed $\varepsilon > 0$. \qed

% \paragraph{What the theorem does and does not claim.}
% This is an existence result: for bounded observations and fixed active counts,
% there is a continuous pooled representation followed by a sufficiently wide
% readout that approximates any continuous invariant policy.  It does not say
% that finite data, a finite network, or a particular optimizer will
% automatically find that readout.  The fixed-cardinality assumption removes an
% ambiguity in the proof; in the implementation, the activity mask and the
% curriculum level provide the corresponding count information, while the same
% shared feature map is used at every level.

\section{Hyperparameters}
\label{sec:hyperparams}

\begin{table}[h] \centering \small \caption{PPO hyperparameters shared across all teacher encoder experiments.} \label{tab:supp_hyperparams} \begin{tabular}{lll} \toprule \textbf{Parameter} & \textbf{Value} & \textbf{Note} \\ \midrule Parallel environments & 4096 & Isaac Lab \\ Total iterations      & 30\,000 & \\ Steps per env per iter & 24 & \\ Mini-batches          & 4 & \\ PPO epochs per iter   & 5 & \\ Learning rate         & $3\times 10^{-4}$ & Adam \\ Discount $\gamma$     & 0.99 & \\ GAE $\lambda$         & 0.95 & \\ Clip range $\epsilon$ & 0.2 & \\ KL target             & $6\times 10^{-3}$ & \\ Entropy coeff.        & $0.01$ & \\ Value loss coeff.     & $1.0$ & \\ Max grad norm         & 1.0 & \\ Control frequency     & \SI{50}{Hz} & \\ History length $H$    & 4 & \\ Max ball count $N$    & 5 & \\ \bottomrule
  \end{tabular}
\end{table}

\begin{table}[h] \centering \small \caption{DAgger distillation hyperparameters (\TactSet). The simulated student and the physical sensor use resolution-aligned $16{\times}16$ binary maps; the physical array is sampled at \SI{10}{Hz}, and training uses a sampled $1$--$2$ frame delay.} \label{tab:supp_distill} \begin{tabular}{ll} \toprule \textbf{Parameter} & \textbf{Value} \\ \midrule Rounds          & 8 \\ BC steps (round 0)       & 400\,000 \\ DAgger steps (rounds 1--7) & 200\,000 each \\ Batch size (training)    & 20\,000 \\ Teacher inference batch  & 2\,048 \\ Learning rate            & $5\times 10^{-4}$ \\ Tactile input shape      & $(2, 16, 16)$ \\ Pre-encoder              & CNN, embedding dim $64$ \\ Temporal encoder         & GRU, hidden size $512$ \\ Delay range (training)   & $1$--$2$ frames \\ Physical array           & $16{\times}16$, \SI{10}{Hz}, \SI{5}{g} threshold \\ \bottomrule
\end{tabular}
\end{table}

The teacher settings use the same PPO budget for every encoder so that changes
in curriculum progression can be attributed to the representation rather than
to a different optimizer schedule.  The student schedule is reported in the
same units used by the data pipeline: one BC initialization round followed by
seven DAgger rounds, with teacher actions queried on the current on-policy
states.  The two tactile channels correspond to the current binary contact map
and its delayed sample, not to two independent sensors.

\section{Environment Specification}
\label{sec:env_spec}

\subsection{Reward Function}
\label{sec:rewards}

The total reward $r_t$ is the sum of object-balancing terms, locomotion terms, and a penalty term:
\begin{equation}
  r_t = r_{\mathrm{obj}} + r_{\mathrm{loco}} + r_{\mathrm{penalty}}.
\end{equation}

The three terms separate the task into the quantities the robot must balance,
the commanded body motion, and regularization costs.  This decomposition is
useful when diagnosing a rollout: a high locomotion score does not compensate
for dropping a ball, and a visually stable plate does not by itself imply that
the commanded velocity was tracked.

\paragraph{Object-balancing rewards.} Let $\mathcal{A}_t \subseteq \{1,\cdots,N\}$ be the set of active ball indices at time $t$.  For each ball $i \in \mathcal{A}_t$, define the \emph{support margin} $\mu_i \in [0,1]$ as the minimum normalized distance from ball $i$'s projected position to the plate boundary (1 = plate center, 0 = edge):
\begin{align}
  r_{\mathrm{margin}} &= +3.0\;\cdot\; \frac{1}{|\mathcal{A}_t|}
    \sum_{i\in\mathcal{A}_t} \mu_i, \\
  r_{\mathrm{tail}} &= +1.25\;\cdot\; \mathrm{CVaR}_{0.4}\!\bigl(\{\mu_i\}_{i\in\mathcal{A}_t}\bigr),
\end{align}
where $\mathrm{CVaR}_{0.4}$~\citep{rockafellar2000cvar} is the conditional value-at-risk at the $0.4$ tail, i.e.\ the conditional mean of the lowest 40\% of margins (a tail-risk reward term incentivising stabilisation of the most precarious ball).
\begin{align}
  r_{\mathrm{edge}} &= -15.0\;\cdot\; \frac{1}{|\mathcal{A}_t|}
    \sum_{i\in\mathcal{A}_t} \mathbf{1}[\mu_i < 0], \\
  r_{\mathrm{spacing}} &= +1.5\;\cdot\;\frac{2}{|\mathcal{A}_t|(|\mathcal{A}_t|-1)}
    \sum_{i<j} \min\!\bigl(\|p_i - p_j\|_2,\, d_{\mathrm{max}}\bigr), \\
  r_{\mathrm{spacing,tail}} &= +0.75\;\cdot\;
    \mathrm{CVaR}_{0.4}\!\bigl(\{\|p_i-p_j\|_2\}_{i<j}\bigr), \\
  r_{\mathrm{dangerous}} &= -20.0\;\cdot\;
    \mathbf{1}\!\bigl[\min_{i\in\mathcal{A}_t}\mu_i < \mu_{\mathrm{thresh}}\bigr],
\end{align}
where $p_i \in \RR^2$ is the ball's projected position on the plate, $d_{\mathrm{max}}$ is a clipping radius, and $\mu_{\mathrm{thresh}}$ is the dangerous-state threshold (set per curriculum level).

The mean-margin term rewards the average safety of the set, while the
CVaR-based term focuses on its weakest members: one ball close to the edge
should not be hidden by four well-centered balls.  The edge and dangerous
terms turn this preference into an immediate cost when a ball crosses a hard
boundary or enters a near-failure band.  Pairwise spacing discourages balls
from collapsing into one another; its clipped distance prevents very distant
pairs from dominating the gradient.  Every sum is normalized by the number of
active balls (or active pairs), so reward magnitudes remain comparable as the
curriculum changes $k$.  When a level has only one active ball, the pairwise
spacing terms are omitted rather than evaluated with a zero denominator.

\paragraph{Velocity-related penalties.}
\begin{align}
  r_{\mathrm{vel},xy} &= -0.05\;\cdot\; \frac{1}{|\mathcal{A}_t|}
    \sum_{i\in\mathcal{A}_t} \|v_i^{xy}\|_2^2, \\
  r_{\mathrm{vel},z}  &= -0.25\;\cdot\; \frac{1}{|\mathcal{A}_t|}
    \sum_{i\in\mathcal{A}_t} |v_i^{z}|,
\end{align}
where $v_i^{xy}$ and $v_i^z$ are the ball's horizontal and vertical velocity in the robot frame.

The two penalties are deliberately asymmetric.  Horizontal motion is allowed
when it is required to carry the robot, whereas vertical motion is a stronger
warning that a ball is bouncing or losing contact with the plate.  They
therefore penalize motion without forcing the controller to freeze the robot.

\paragraph{Locomotion rewards (from base locomotion policy).}
\begin{align}
  r_{\mathrm{track},v}  &= +1.5\;\cdot\;
    \exp\!\bigl(-\|\hat{v}^{xy} - v^{xy}_{\mathrm{cmd}}\|_2^2 / 0.25\bigr), \\
  r_{\mathrm{track},\omega} &= +1.25\;\cdot\;
    \exp\!\bigl(-|\hat{\omega}_z - \omega_{z,\mathrm{cmd}}|^2 / 0.25\bigr), \\
  r_{\mathrm{height}} &= -1.0\;\cdot\; (h - h_{\mathrm{ref}})^2, \\
  r_{\mathrm{roll,pitch}} &= -1.25\;\cdot\; \|\theta_{\mathrm{rp}}\|_2^2,
\end{align}
where $\hat{v}^{xy}$ and $\hat{\omega}_z$ are the robot's base velocity, $v^{xy}_{\mathrm{cmd}}$ and $\omega_{z,\mathrm{cmd}}$ are the commanded velocities, $h$ is the base height, $h_{\mathrm{ref}}$ is the target height, and $\theta_{\mathrm{rp}}$ are the roll and pitch angles.

The exponential tracking terms are one at perfect tracking and decay smoothly
as the error grows.  Height and roll/pitch penalties provide a soft preference
for an upright base, while the explicit termination rules below handle events
that are already outside the useful training region.

\subsection{Observation Space}
\label{sec:obs}

The teacher policy receives the following observations at each timestep:

\paragraph{Proprioception (history length $H_{\mathrm{prop}}=6$, distinct from the object-history length $H=4$).} \begin{itemize}[topsep=2pt,itemsep=0pt,leftmargin=*]
\item Base angular velocity $\omega_b \in \RR^3$
\item Projected gravity vector $g_b \in \RR^3$
\item Velocity command $[v_x, v_y, \omega_z] \in \RR^3$
\item Joint positions $q \in \RR^{12}$ (relative to default)
\item Joint velocities $\dot{q} \in \RR^{12}$
\item Previous action $a_{t-1} \in \RR^{12}$
\item Robot base height $h \in \RR^1$
\end{itemize}
Total proprioception dimension: $43 \times H_{\mathrm{prop}} = 258$.

\paragraph{Multi-ball state (object-history length $H=4$, per-ball feature dim $d=14$).} For each active ball $i \in \mathcal{A}_t$: \begin{itemize}[topsep=2pt,itemsep=0pt,leftmargin=*]
\item Relative position $\Delta p_i \in \RR^3$ (ball position minus robot base, in robot frame; scale $1.0$)
\item Relative linear velocity $\Delta v_i \in \RR^3$ (scale $0.5$)
\item Relative orientation as quaternion $q_i \in \RR^4$ (scale $1.0$)
\item Relative angular velocity $\omega_i \in \RR^3$ (scale $0.25$)
\item Activity flag $a_i \in \{0,1\}$ (scale $1.0$)
\end{itemize}
Inactive slots are zero-masked.  Observation noise (uniform): position $\pm 0.01\,\text{m}$, velocity $\pm 0.2\,\text{m/s}$, orientation $\pm 0.05$ (Euler-angle equivalent), angular velocity $\pm 0.2\,\text{rad/s}$. Ball state is observable only when the ball is in contact with the robot surface (contact detection threshold $\approx 10^{-8}\,\text{N}$); otherwise the observation is zero.

The critic (teacher) additionally receives the denoised version of the above (no observation noise added).

The actor is therefore trained and evaluated with the information quality that
the physical sensor can provide: recent proprioceptive history, a short
object history, masking for missing contacts, and bounded noise.  The critic's
clean copy is a training aid for estimating returns; it is not an extra input
available to the deployed policy or to the tactile student.  Keeping this
distinction explicit prevents the value function from being mistaken for a
privileged sensing channel.

\subsection{Termination Conditions}
\label{sec:terminations}

An episode terminates (non-timeout) when any of the following occurs: \begin{enumerate}[topsep=2pt,itemsep=0pt,leftmargin=*]
\item \textbf{Ball falls below robot base}: any active ball's $z$-position drops below the robot base $z$-position, $\exists i \in \mathcal{A}_t: p_{i,z} < p_{\mathrm{robot},z}$.
\item \textbf{Robot falls}: the projected gravity vector indicates a roll angle $|\arcsin(g_{b,y})| > 90^{\circ}$.
\item \textbf{Base height too low}: the robot base height falls below the minimum standing height threshold.
\end{enumerate}
Note: the \texttt{base\_contact} termination (body contact with the ground) is disabled for the multi-sphere task to prevent false positives from the tactile plate surface.

These checks have complementary purposes.  A ball-below-base event marks a
clear loss of the payload, the projected-gravity test catches an inverted or
rolling robot, and the height threshold rejects crouched configurations that
could otherwise collect short-lived balancing reward.  Disabling
\texttt{base\_contact} is specific to this setup: the instrumented plate can
touch the body during normal operation and would otherwise terminate healthy
episodes.

\section{Curriculum Design}
\label{sec:curriculum}

Curriculum learning~\citep{bengio2009curriculum,narvekar2020curriculum} progresses the agent from easy to hard task instances. In our setting the difficulty axis is the active ball count $k$.

\subsection{Promotion Criteria}

The multi-ball curriculum promotes from $k$ to $k+1$ active balls when a rolling window of recent episodes satisfies \emph{all six} of the following conditions simultaneously: \begin{enumerate}[topsep=2pt,itemsep=0pt,leftmargin=*]
\item \textbf{Episode-length ratio} $\ge 0.85$: the mean episode length divided by the maximum episode length exceeds 85\%, indicating the agent consistently survives to episode end.
\item \textbf{Support margin} $\ge \mu^*_k$: mean minimum support margin across active balls exceeds the level threshold.
\item \textbf{Dangerous fraction} $\le f^*_{d,k}$: the fraction of timesteps with any ball in a dangerous state (near edge) is below the level threshold.
\item \textbf{Edge-overflow fraction} $\le f^*_{e,k}$: the fraction of timesteps with any ball outside the plate boundary is below the level threshold.
\item \textbf{Linear velocity tracking error} $\le \epsilon^*_{v,k}$: mean $\|\hat{v}^{xy} - v^{xy}_{\mathrm{cmd}}\|_2$ is below threshold.
\item \textbf{Angular velocity tracking error} $\le \epsilon^*_{\omega,k}$: mean $|\hat{\omega}_z - \omega_{z,\mathrm{cmd}}|$ is below threshold.
\end{enumerate}

All conditions must hold simultaneously for \texttt{required\_successes}$\,=\,1$ consecutive evaluation windows, with \texttt{min\_stage\_episodes}$\,=\,2048$ and \texttt{min\_steps\_per\_level}$\,=\,3200$.

The promotion rule is an AND gate rather than an average score.  An agent can
survive for a long time while repeatedly allowing one ball to approach the
edge, or track its velocity while losing the payload; either behavior should
not unlock a harder level.  We evaluate the conditions over a rolling window
to smooth individual simulator episodes, require a minimum amount of data at
each level, and use one consecutive successful window so that promotion is
not delayed after a genuine plateau is reached.

\subsection{Per-Level Promotion Thresholds}

\begin{table}[h] \centering \small \caption{Per-level curriculum promotion thresholds ($k \to k+1$). ``Support margin'' $\mu^*$ is the normalized minimum distance from plate boundary. ``Dangerous frac.'' $f^*_d$ and ``Edge frac.'' $f^*_e$ are allowed fractions of unsafe timesteps. ``Vel.\ err.'' $\epsilon^*_v$ and ``Yaw err.'' $\epsilon^*_\omega$ are maximum allowed tracking errors.} \label{tab:curriculum_thresholds} \begin{tabular}{cccccc} \toprule \textbf{Level ($k\to k+1$)} & $\mu^*$ & $f^*_d$ & $f^*_e$ & $\epsilon^*_v$ & $\epsilon^*_\omega$ \\ \midrule $1 \to 2$ & 0.30 & 0.20 & 0.02 & 0.30 & 0.40 \\ $2 \to 3$ & 0.25 & 0.25 & 0.03 & 0.35 & 0.45 \\ $3 \to 4$ & 0.20 & 0.30 & 0.04 & 0.40 & 0.50 \\ $4 \to 5$ & 0.15 & 0.35 & 0.05 & 0.45 & 0.55 \\ \bottomrule
  \end{tabular}
\end{table}

The thresholds are relaxed progressively at higher levels to account for the increased difficulty of balancing more balls simultaneously.

In robot terms, the curriculum asks for strict safety at low ball counts and
then permits a controlled amount of degradation as capacity is approached.
The safety thresholds and motion-tracking thresholds are relaxed together so
that the policy is not rewarded for trading all payload margin for locomotion
accuracy.  The traces in \cref{fig:curriculum_conditions} make the promotion
decision auditable: every colored curve must cross its corresponding dashed
line before the grey stage marker is reached.

\begin{figure}[h] \centering \includegraphics[width=0.85\linewidth]{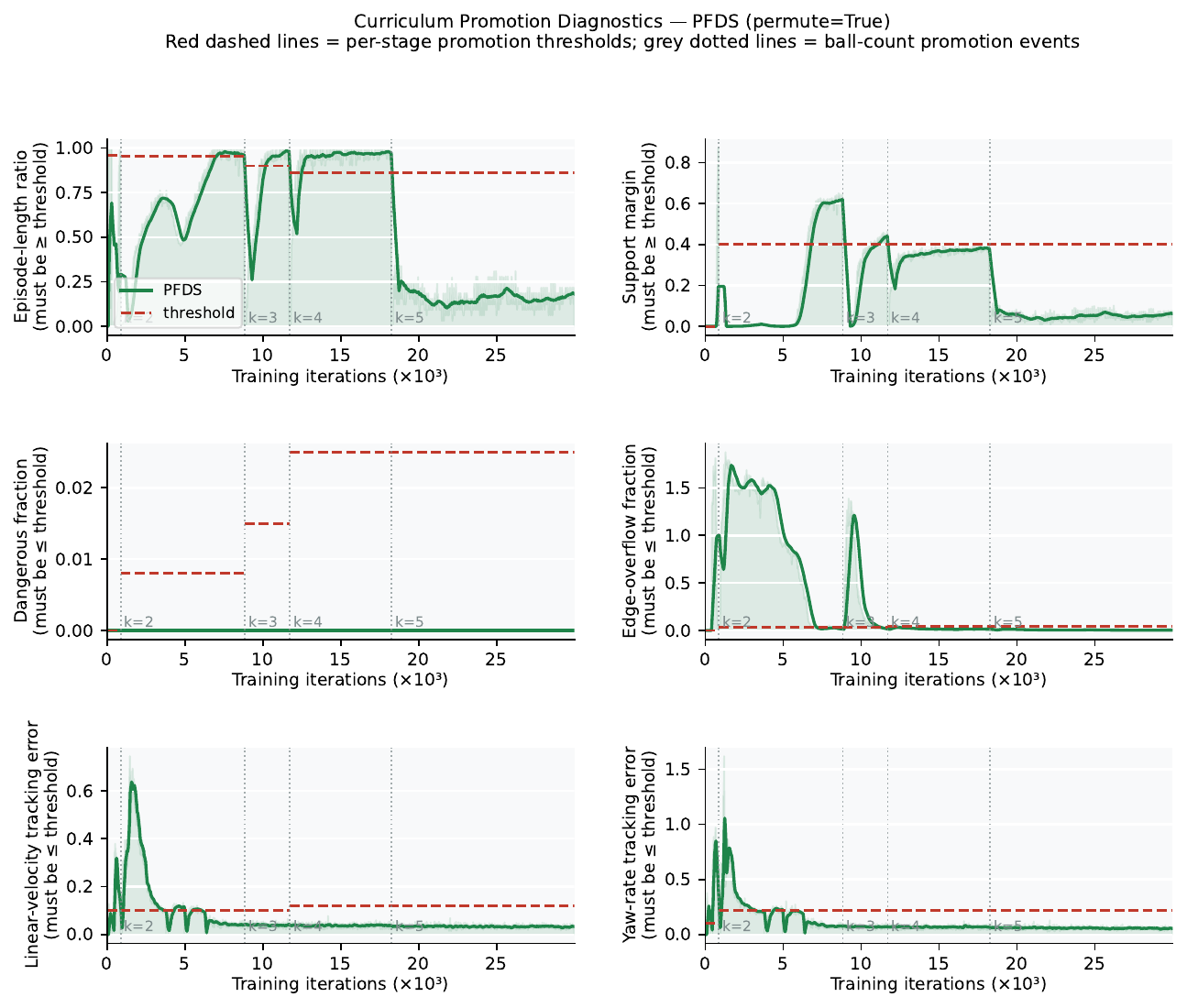} \caption{Curriculum promotion-criterion traces for a representative \PFDS training run. All six promotion conditions are shown: episode-length ratio, support margin, dangerous-state fraction, edge-overflow fraction, linear-velocity tracking error, and yaw-rate tracking error. Red dashed lines are per-stage promotion thresholds; grey dotted lines mark the iterations at which the curriculum advances to the next level. All six conditions must simultaneously meet threshold for promotion.} \label{fig:curriculum_conditions}
\end{figure}

% ============================================================================
\section{Discussion}
\label{sec:discussion}
% ============================================================================

\paragraph{Architectural vs.\ augmentation guarantees.} In the main \DSHCOFF run, the failure behaves like a budget-limited barrier: slot identity becomes a deterministic proxy for curriculum stage and blocks promotion beyond $k{=}2$ for the full 30\,000-iteration budget. The additional-seed study in \cref{app:postreview} shows that this is not irreversible: some seeds do eventually advance, but they do so late in training and reach substantially lower strict success. The architectural reading is therefore conservative: without a $\Gframe$-invariant encoder, curriculum progression is slow, seed-dependent, and low-yield. \PFDS avoids this by construction; without augmentation it advances to $k{=}5$ but is slower to reach $k{\ge}3$ ($\sim$16.9k vs.\ $\sim$3.9k iterations, \cref{tab:perm_ablation}).

\paragraph{Curriculum-dynamics mechanism of \DSHC failure.} Beyond the symmetry-group argument, the failure has a concrete curriculum-dynamics signature. When slot-permutation augmentation is disabled, slot index $i$ is consistently occupied by ball $i$ across an episode; because the curriculum activates balls sequentially, the slot index becomes a deterministic proxy for the current curriculum stage. The policy keys on slot identity rather than physical state, and fails to satisfy the \emph{conjunction} of promotion criteria at $k{=}2$: the episode-length ratio remains high (the agent does not fail outright) but the support margin never reaches the promotion threshold. \PFDS is immune by construction, since its per-frame pooling discards slot-identity information. \Cref{fig:diagnostic} reports the resulting traces in the runs we observed.

\begin{figure}[h] \centering \includegraphics[width=0.85\linewidth]{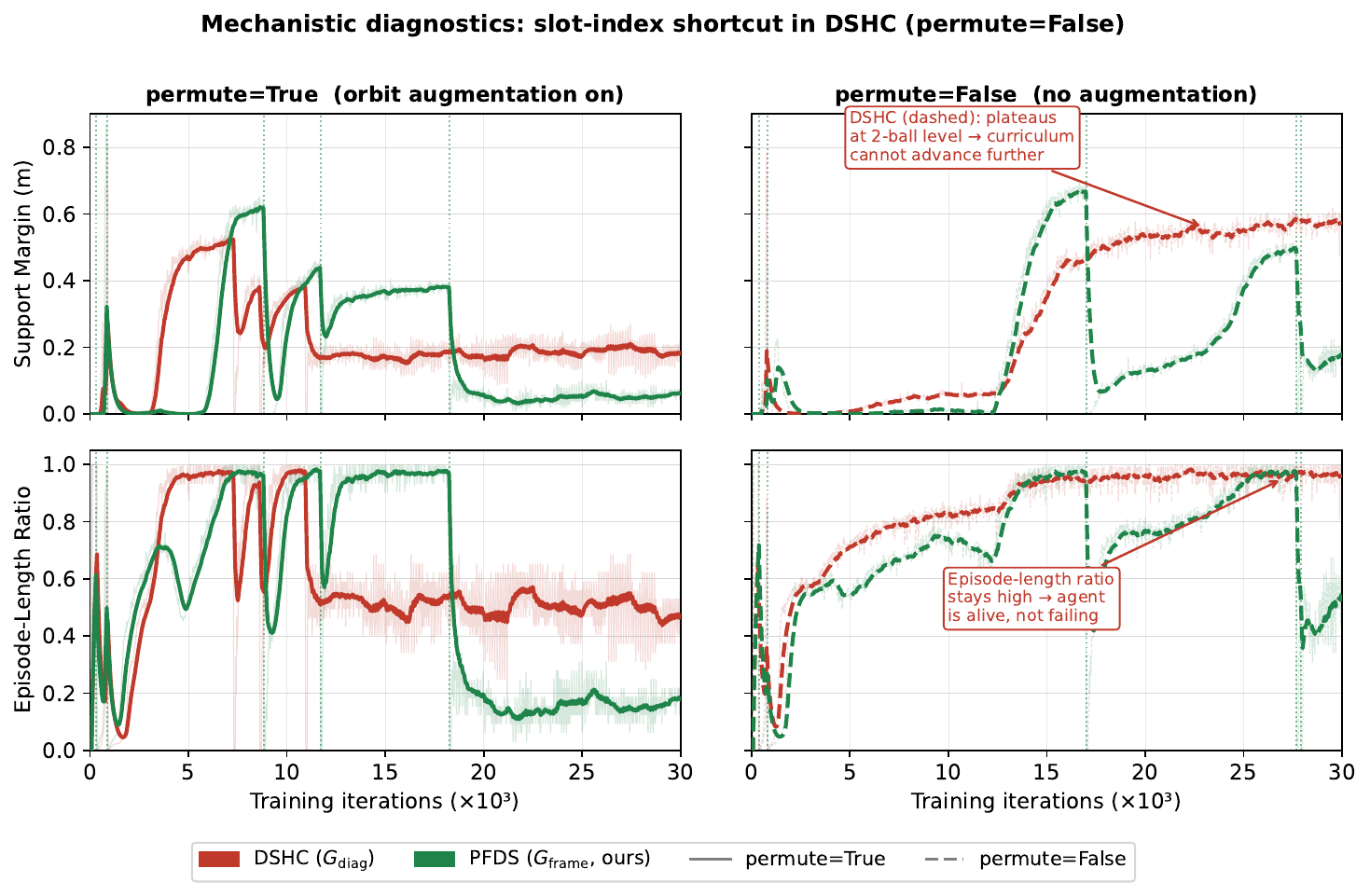} \caption{Curriculum-progress diagnostic for the slot-index shortcut. Rows show two promotion-related signals (top: support margin; bottom: episode-length ratio). Columns compare slot-permutation augmentation perm.\ True (left) and perm.\ False (right). \DSHC (dashed) and \PFDS (solid) are overlaid. Under perm.\ False, \DSHCOFF keeps the episode-length ratio high but cannot push the support margin to its promotion threshold, so the curriculum stalls at $k{=}2$; \PFDSOFF reaches both thresholds and advances. Under perm.\ True, both encoders meet the conjunction of promotion criteria.} \label{fig:diagnostic}
\end{figure}

\paragraph{Sample-efficiency cost of disabling augmentation.} Without slot-permutation augmentation, the training distribution samples only the identity element of $\Gframe$ at each step, reducing effective distributional diversity. Since \PFDS is $\Gframe$-invariant by construction, this does not prevent convergence but does increase the number of iterations required to cover the state space (\cref{tab:perm_ablation}: $\sim$16.9k vs.\ $\sim$3.9k iterations to $k{\ge}3$).

% ============================================================================
\section{Additional Experiments and Diagnostics}
\label{app:postreview}
% ============================================================================

This appendix documents additional experiments on encoder robustness, curriculum dynamics, and distillation. Unless stated otherwise, all values follow the evaluation protocol of \cref{sec:problem}: 100 episodes per cell, $\pm0.25$\,m/s command range, strict success as defined in \cref{sec:problem}, and no-drop relaxed to completion of the full horizon without a sphere leaving the plate. Additional training seeds are reported for the encoder most affected by the symmetry mismatch (\DSHCOFF) and for \PFDS; a full multi-seed sweep of every augmented baseline remains future work. We therefore report the unaugmented \DSHC result as seed-limited rather than as a universal failure. The tactile experiments isolate the representation (binary union map) and the distillation procedure (BC vs.\ DAgger), but do not sweep CNN/GRU architecture, history length, or resolution; those remain future work. The remainder is organized as follows: additional \DSHCOFF seeds (\cref{app:dshc-seeds}); slot relabeling and action-level invariance (\cref{app:relabel}); fixed-$k{=}5$ policies without curriculum (\cref{app:fixedk}); \TactSet robustness (\cref{app:tactset-robustness}); DAgger versus behavior cloning (\cref{app:dagger}); transport-quality distributions (\cref{app:transport-quality}); and sim-to-sim and set-size extrapolation (\cref{app:sim2sim}).

\paragraph{How to read the diagnostics.}
The tables intentionally separate strict success from no-drop success.  The
former also requires the commanded motion and support-margin criteria, whereas
the latter asks only whether the robot completes the horizon without losing a
sphere.  A gap between the two is therefore evidence of near-edge or tracking
failures rather than an immediate payload drop.  Unless a caption states
otherwise, percentages are episode rates, not per-timestep averages.

\subsection{Additional \DSHCOFF Seeds}
\label{app:dshc-seeds}

\begin{table}[!ht] \centering \small \caption{Additional \DSHCOFF training seeds under the training setup of \cref{sec:exp}. The originally trained seed (42) stalls at $k{=}2$; the other seeds eventually advance, but late in training and with much lower strict success than \PFDS. Evaluation uses the protocol of \cref{sec:problem} (perm.\ True) and a matched perm.\ False pass. Values at $k{=}5$ are percent.} \label{tab:post-dshc-seeds}
\setlength{\tabcolsep}{4pt}
\begin{tabular}{lccccc}
\toprule
\textbf{Seed} & \textbf{Curriculum} & \multicolumn{2}{c}{\textbf{Permute on}} & \multicolumn{2}{c}{\textbf{Permute off}} \\
\cmidrule(lr){3-4}\cmidrule(lr){5-6}
 & & strict & no-drop & strict & no-drop \\
\midrule
42 (main run) & stalls at $k{=}2$, all 30k & -- & -- & -- & -- \\
41 & $k{=}5$ at ${\sim}22.4$k & 22 & 87 & 70 & 99 \\
43 & $k{=}5$ at ${\sim}15$k & 23 & 100 & 79 & 100 \\
\PFDS (reference) & $k{=}5$ at ${\sim}6.3$k & 39--41 & 98--100 & -- & -- \\
\bottomrule
\end{tabular}
\end{table}

\subsection{Slot Relabeling and Action-Level Invariance}
\label{app:relabel}

\begin{table}[!ht] \centering \small \caption{Test-time per-frame slot relabeling. Each cell is strict/no-drop percent, re-evaluated with our evaluation harness ($n{=}100$). Absolute values can differ slightly from \cref{tab:encoders} because of simulator-asset drift; the invariant quantity is the within-encoder native-to-stress change. \DSHCOFF at $k{=}5$ is constant (0/0) because the originally trained seed (42) never reaches that level.} \label{tab:post-relabel}
\setlength{\tabcolsep}{3pt}
\begin{tabular}{lcccc}
\toprule
\textbf{Encoder} & \multicolumn{2}{c}{\textbf{$k{=}2$}} & \multicolumn{2}{c}{\textbf{$k{=}5$}} \\
\cmidrule(lr){2-3}\cmidrule(lr){4-5}
 & native & relabeled & native & relabeled \\
\midrule
\PFDS       & 90/93 & 91/94 & 38/97 & 39/99 \\
\PFDSOFF    & 98/100 & 96/99 & 24/91 & 42/92 \\
\DSHC       & 95/98 & 99/99 & 53/99 & 72/100 \\
\DSHCOFF    & 91/94 & 53/81 & 0/0 & -- \\
\bottomrule
\end{tabular}
\end{table}

\begin{table}[!ht] \centering \small \caption{Action-level invariance gap under $G_{\mathrm{frame}}$ permutations, reported as a fraction of the policy's action std (mean over 200 states $\times$ 8 permutations per sphere count; lower is better). Both models are exactly $G_{\mathrm{diag}}$-invariant at numerical-noise level ($4$--$7{\times}10^{-4}$), so the comparison isolates their residual frame-wise dependence.} \label{tab:post-gap}
\begin{tabular}{lcccc}
\toprule
\textbf{Model} & $k{=}2$ & $k{=}3$ & $k{=}4$ & $k{=}5$ \\
\midrule
\PFDS & $6.4{\times}10^{-4}$ & $5.2{\times}10^{-4}$ & $4.0{\times}10^{-4}$ & $3.6{\times}10^{-4}$ \\
\DSHC (perm.\ True) & $2.02{\times}10^{-1}$ & $1.54{\times}10^{-1}$ & $1.23{\times}10^{-1}$ & $9.5{\times}10^{-2}$ \\
Ratio & $316\times$ & $298\times$ & $307\times$ & $265\times$ \\
\bottomrule
\end{tabular}
\end{table}

\begin{figure}[t]
  \centering
  \includegraphics[width=0.98\linewidth]{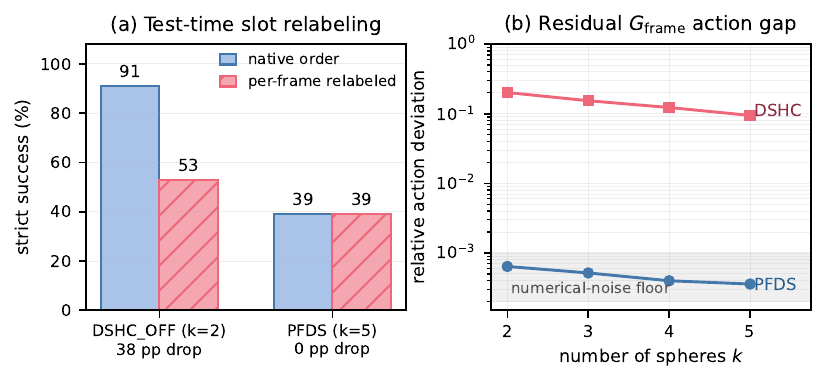}
  \caption{Symmetry diagnostics. \textbf{(a)} Test-time per-frame slot relabeling at $k{=}2$ reduces the \DSHCOFF policy's strict success from $91\%$ to $53\%$ ($-38$\,pp), while \PFDS at $k{=}5$ is unchanged. \textbf{(b)} Relative action deviation under $G_{\mathrm{frame}}$ permutations (lower is better; log scale). Both encoders are numerically $G_{\mathrm{diag}}$-invariant, but augmentation-trained \DSHC remains ${\sim}265$--$316\times$ larger than \PFDS, showing that augmentation yields only approximate frame-wise invariance.}
  \label{fig:symmetry_diagnostics}
\end{figure}

\subsection{Fixed-\texorpdfstring{$k{=}5$}{k=5} Policies Without Curriculum}
\label{app:fixedk}

\begin{table}[!ht] \centering \small \caption{Fixed-$k{=}5$ specialists trained without curriculum for 15k iterations. Both encoders learn five-sphere transport, showing that the low-$k$ curriculum plateau is not an architectural capacity limit. However, the fixed-$k$ experts score 0\% at every $k{=}1,\cdots,4$, whereas the curriculum policy works across all counts.} \label{tab:post-fixedk}
\setlength{\tabcolsep}{4pt}
\begin{tabular}{lcccc}
\toprule
\textbf{Model} & \multicolumn{2}{c}{\textbf{$k{=}5$ strict/no-drop}} & \multicolumn{2}{c}{\textbf{$k{=}1,\cdots,4$}} \\
\cmidrule(lr){2-3}\cmidrule(lr){4-5}
 & perm.\ True & perm.\ False & strict & no-drop \\
\midrule
\DSHCOFF & 73/98 & 70/99 & 0 & 0 \\
\PFDSOFF  & 21/100 & 32/100 & 0 & 0 \\
\bottomrule
\end{tabular}
\end{table}

\subsection{\TactSet Robustness}
\label{app:tactset-robustness}

\begin{table}[!ht] \centering \small \caption{Zero-shot \TactSet robustness on the shipped student, re-evaluated with our evaluation harness. Each cell is no-drop/strict percent. The unperturbed $k{=}5$ no-drop rate is 85\% in this harness; the originally reported 75\% differs by simulator-asset generation, so all comparisons use the matched baseline.} \label{tab:post-tactset}
\setlength{\tabcolsep}{4pt}
\begin{tabular}{lcc}
\toprule
\textbf{Perturbation} & \textbf{$k{=}3$} & \textbf{$k{=}5$} \\
\midrule
None (baseline) & 100/75 & 85/17 \\
Taxel dropout 1\% & 92/59 & 84/13 \\
Taxel dropout 5\% & 97/72 & 84/11 \\
Taxel dropout 10\% & 95/73 & 84/18 \\
Taxel dropout 20\% & 71/42 & 61/11 \\
Spurious contact 1\% & 90/61 & 68/12 \\
Spurious contact 5\% & 92/52 & 76/18 \\
Calibration shift, row & 87/43 & 86/17 \\
Calibration shift, column & 93/32 & 61/2 \\
Fixed 2-frame latency & 90/49 & 71/7 \\
Fixed 4-frame latency & 77/30 & 48/1 \\
\bottomrule
\end{tabular}
\end{table}

\subsection{DAgger Versus Behavior Cloning}
\label{app:dagger}

\begin{table}[!ht] \centering \small \caption{DAgger round-by-round evaluation of intermediate students, reported as no-drop/strict percent. The round-7 row is the shipped checkpoint evaluated in the \TactSet robustness harness (the unperturbed row of \cref{tab:post-tactset}).} \label{tab:post-dagger}
\begin{tabular}{lccc}
\toprule
\textbf{Round} & \textbf{$k{=}1$} & \textbf{$k{=}3$} & \textbf{$k{=}5$} \\
\midrule
0 (BC only) & 1/0 & 4/0 & 4/0 \\
1 (first on-policy) & 93/70 & 67/34 & 72/14 \\
3 & 74/50 & 93/42 & 77/15 \\
5 & 77/52 & 89/43 & 93/20 \\
7 (shipped) & -- & 100/75 & 85/17 \\
\bottomrule
\end{tabular}
\end{table}

The gain is concentrated in the first on-policy round and saturates thereafter, which supports a representation-limited rather than data-limited teacher--student gap.

\begin{figure}[h]
  \centering
  \includegraphics[width=0.98\linewidth]{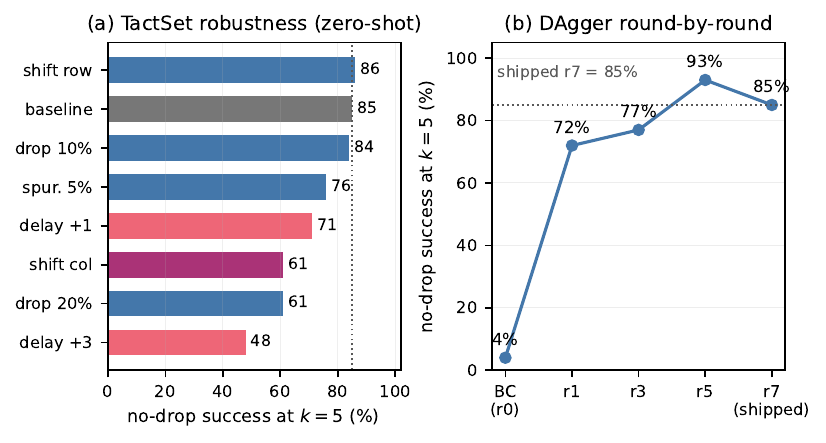}
  \caption{\TactSet robustness and DAgger. \textbf{(a)} Zero-shot no-drop success at $k{=}5$ under binary contact-map perturbations, sorted by severity (red: latency; purple: calibration shift); latency and column calibration shift are the dominant sensitivities. \textbf{(b)} DAgger/BC round-by-round no-drop success at $k{=}5$. Pure BC (round 0) reaches $4\%$; one on-policy round lifts this to $72\%$, and the shipped round-7 checkpoint measures $85\%$ in the same harness.}
  \label{fig:tactset_diagnostics}
\end{figure}

\subsection{Transport-Quality Distributions}
\label{app:transport-quality}

\begin{table}[!ht] \centering \small \caption{Fine-grained transport-quality statistics at $k{=}5$ over 100 teacher episodes (1000 steps). Margins are minimum support margin percentiles; near-edge dwell is the mean fraction of steps with a near-edge event; failure decomposition separates strict failures caused only by a transient near-edge excursion from failures that also lost a sphere.} \label{tab:post-quality}
\setlength{\tabcolsep}{3pt}
\begin{tabular}{lccccc}
\toprule
\textbf{Model} & \textbf{margin P10/50/90} & \textbf{drops} & \textbf{time-to-drop} & \textbf{near-edge} & \textbf{near-only / dropped} \\
\midrule
\PFDS & 0.000/0.000/0.061 & 0/100 & -- & 0.1\% & 59 / 0 \\
\DSHC & 0.000/0.029/0.100 & 2/100 & 90 & 0.4\% & 30 / 2 \\
\PFDSOFF & 0.000/0.000/0.085 & 2/100 & 60 & 0.6\% & 57 / 2 \\
ST (PF) & 0.000/0.009/0.182 & 0/100 & -- & 0.1\% & 46 / 0 \\
\bottomrule
\end{tabular}
\end{table}

\paragraph{Margin objective.} The teacher reward already contains four margin-related terms: mean support margin, a CVaR$_{0.4}$ tail bonus, a near-edge penalty, and an out-of-zone penalty (\cref{sec:rewards}). Strengthening these terms further trades off against velocity tracking, whose error already grows with $k$; a systematic weight sweep and direct distributional margin objectives are left for future work.

\subsection{Sim-to-Sim and Set-Size Extrapolation}
\label{app:sim2sim}

We export the distilled policy through the deployment ONNX stack and evaluate it in MuJoCo, a different contact-solver family, on the same robot equipped with twice the tactile sensing area (two $16{\times}16$ panels, represented as a $32{\times}16$ map) and \SI{2}{cm}-radius steel spheres. The robot survives in all 250 randomized episodes ($k{=}1,\cdots,5$ crossed with stand, three forward speeds, and a turn command, ten seeds each). Individual sphere retention is $97.7\%$; $236/250$ episodes ($94.4\%$) retain every sphere, with the loss concentrated at $k{=}4$ (the four-ball spawn is less self-bracing than the five-ball packing); the robot itself never falls. \Cref{tab:post-sim2sim} summarizes the per-$k$ result. Policies trained up to $k{=}5$ also extrapolate zero-shot to larger sets never seen in training: the teacher retains every sphere in $98\%$, $95\%$, and $93\%$ of $k{=}6,7,8$ episodes, while the tactile student degrades gracefully to $80\%$, $59\%$, and $33\%$.

\begin{table}[!ht] \centering \small \caption{MuJoCo sim-to-sim evaluation of the distilled student on the same robot with twice the tactile sensing area. ``Full retention'' counts episodes in which every sphere remains on the plate; the robot survived in all 50 episodes per count.} \label{tab:post-sim2sim}
\begin{tabular}{lcc}
\toprule
\textbf{Spheres} & \textbf{Full-retention episodes} & \textbf{Rate} \\
\midrule
1 & 50/50 & 100\% \\
2 & 50/50 & 100\% \\
3 & 50/50 & 100\% \\
4 & 38/50 & 76\% \\
5 & 48/50 & 96\% \\
\midrule
Total & 236/250 & 94.4\% \\
\bottomrule
\end{tabular}
\end{table}

% ============================================================================
\section{Physical Platform and Measurement Protocol}
\label{app:hardware}
% ============================================================================

\begin{table}[!ht] \centering \small \caption{Controlled physical sweep. Each run is a single qualitative closed-loop demonstration; duration is the analyzed robot-active window. ``Peak taxels'' and ``blob count'' are measured from the binary $16{\times}16$ map at the selected still frame in \cref{fig:hw}.} \label{tab:hardware-runs}
\begin{tabular}{cccc}
\toprule
\textbf{Spheres} & \textbf{Duration (s)} & \textbf{Peak taxels} & \textbf{8-connected blobs} \\
\midrule
1 & 110.0 & 2 & 1 \\
2 & 115.9 & 4 & 2 \\
3 & 117.1 & 4 & 3 \\
4 & 116.6 & 7 & 4 \\
5 & 119.2 & 8 & 5 \\
\bottomrule
\end{tabular}
\end{table}

Measurement notes. The sensor publishes a pre-thresholded binary frame at \SI{10}{Hz} (\SI{5}{g} threshold), so no continuous pressure or heat-map interpolation is available or claimed. Video and sensor streams are aligned using the bag header timestamp and the video creation time, with a $+0.5$\,s constant calibrated to the first ball--plate contact; the residual synchronization error is below \SI{0.5}{s}, and we do not claim millisecond synchronization or taxel-level camera registration. The blob-count claim is verified by 8-connected component counting on the measured maps in selected frames. We do not report hardware success rates because repeating the same initial condition at scale is not currently feasible.

The hardware table should consequently be read as a measurement audit, not as
an averaged benchmark.  ``Peak taxels'' counts the largest number of active
binary cells in the selected still frame, while the blob count is a descriptive
check that the map contains the expected number of separated contact regions.
Neither statistic estimates contact force, object identity, or a probability
of successful transport.  This distinction keeps the qualitative hardware
demonstration consistent with the quantitative simulation protocol.

% ============================================================================
\section{Extensions: Equivariant Outputs and Heterogeneous Objects}
\label{app:extensions}
% ============================================================================

\PFDS is designed for a global robot action, where permutation \emph{invariance} is the correct symmetry. For object-indexed outputs, the same per-frame construction supports within-frame permutation \emph{equivariance}: retain each pre-pooling embedding $\phi(x_{t,i})$, pair it with the invariant frame context $h_t$, and decode each object's output from $[\phi(x_{t,i});h_t]$. This is the per-frame analogue of the Deep Sets equivariant layer~\citep{zaheer2017deepsets}. Cross-frame object-indexed outputs are fundamentally different: the $G_{\mathrm{frame}}$ quotient identifies histories that differ by independent per-frame relabelings, so persistent cross-frame correspondence is not identifiable without an additional assumption. An explicit matching head or a temporal association module must supply that hypothesis before an equivariant decoder can produce consistent object-indexed trajectories. Global-action transport lies in the identifiable regime, which is why invariance suffices here.

For heterogeneous objects, size, mass, friction, or shape attributes enter as per-object features and break the symmetry only in the data; relabeling invariance remains correct and lossless because object identity is still not part of the observation. Separate per-type Deep Sets blocks are required only when the output itself is type-indexed, e.g., when each object type has a different action head or target assignment. The curriculum and reward remain valid for mixed-size sets as long as the support-margin and termination geometry are recomputed from the per-object radius. Systematic heterogeneous experiments are left to future work.

For a concrete distinction, a single global body-velocity command is unchanged
when two balls swap slots, so invariance is the desired output behavior.  A
controller that predicts one force for each ball instead needs equivariance: the
two predicted forces should swap along with the two input slots.  If a ball is
missing in one frame, however, no mathematically valid decoder can reconstruct
which earlier ball it was without an additional association assumption.  This
is why the present task uses global actions and treats persistent identity as
outside the observable state.

% ============================================================================
\section{Full Per-\texorpdfstring{$k$}{k} Success Rates}
\label{app:perk}
% ============================================================================

\begin{table}[!ht] \centering \small \caption{Per-$k$ strict/no-drop success rates (percent) for the main teacher encoders and the distilled student, evaluated with the released checkpoints and the protocol of \cref{sec:problem}. Teacher values use 1000-step episodes; \TactSet uses 500-step episodes and the original simulator-asset generation, whose matched-harness baseline is 85\% (\cref{tab:post-tactset}).} \label{tab:perk}
\setlength{\tabcolsep}{4pt}
\begin{tabular}{lccccc}
\toprule
\textbf{Method} & $k{=}1$ & $k{=}2$ & $k{=}3$ & $k{=}4$ & $k{=}5$ \\
\midrule
\FlatMLP & 0/0 & 0/0 & 0/0 & 0/0 & 0/0 \\
\BranchMLP & 96/96 & 0/2 & 0/0 & 0/0 & 0/0 \\
\PFDS & 99/99 & 97/99 & 99/100 & 93/99 & 41/100 \\
\PFDSOFF & 99/99 & 99/100 & 90/100 & 77/100 & 41/98 \\
\DSHC & 99/99 & 99/100 & 100/100 & 88/99 & 68/98 \\
\DSHCOFF & 100/100 & 50/78 & 5/13 & 0/0 & 0/0 \\
Set Transformer & 100/100 & 82/98 & 69/98 & 47/100 & 10/95 \\
ST (PF) & 98/99 & 97/98 & 93/100 & 82/100 & 54/100 \\
\TactSet (student) & 92/97 & 72/85 & 72/93 & 50/85 & 19/75 \\
\bottomrule
\end{tabular}
\end{table}

% ============================================================================
\section{Simulation Qualitative Rollouts}
\label{app:qualitative}
% ============================================================================

The main text shows the representative simulation-rollout and physical-platform
views (\cref{fig:sim-rollouts,fig:hardware-stack}). This appendix section
retains the complementary seed and student-result figures below.

\begin{center}
  \centering
  \includegraphics[width=0.48\linewidth]{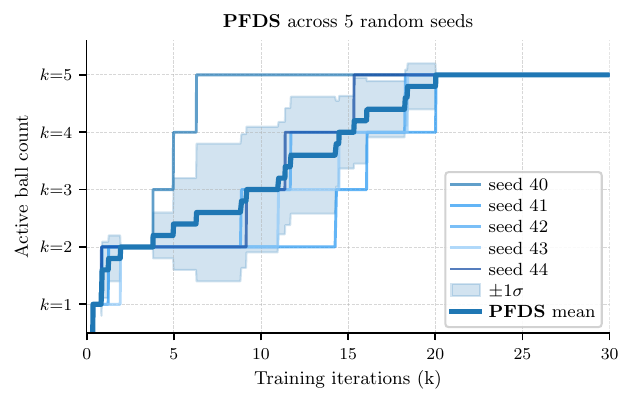}
  \captionof{figure}{Multi-seed \PFDS curriculum progression (mean $\pm 1\sigma$, seeds 40--44). All five seeds reach $k_{\max}{=}5$; the band reflects convergence-timing variation, not final performance.}
  \label{fig:seed_variance_app}
\end{center}

\begin{center}
  \centering
  \includegraphics[width=0.48\linewidth]{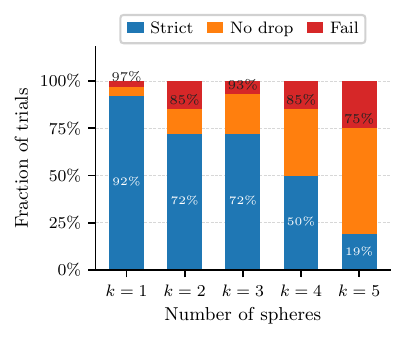}
  \captionof{figure}{Original \TactSet success rates (100 trials per $k$, 500-step episodes; bar top: no-drop; blue inner bars: strict).}
  \label{fig:tactset_results_app}
\end{center}

\begin{center}
  \centering
  \includegraphics[width=0.98\linewidth]{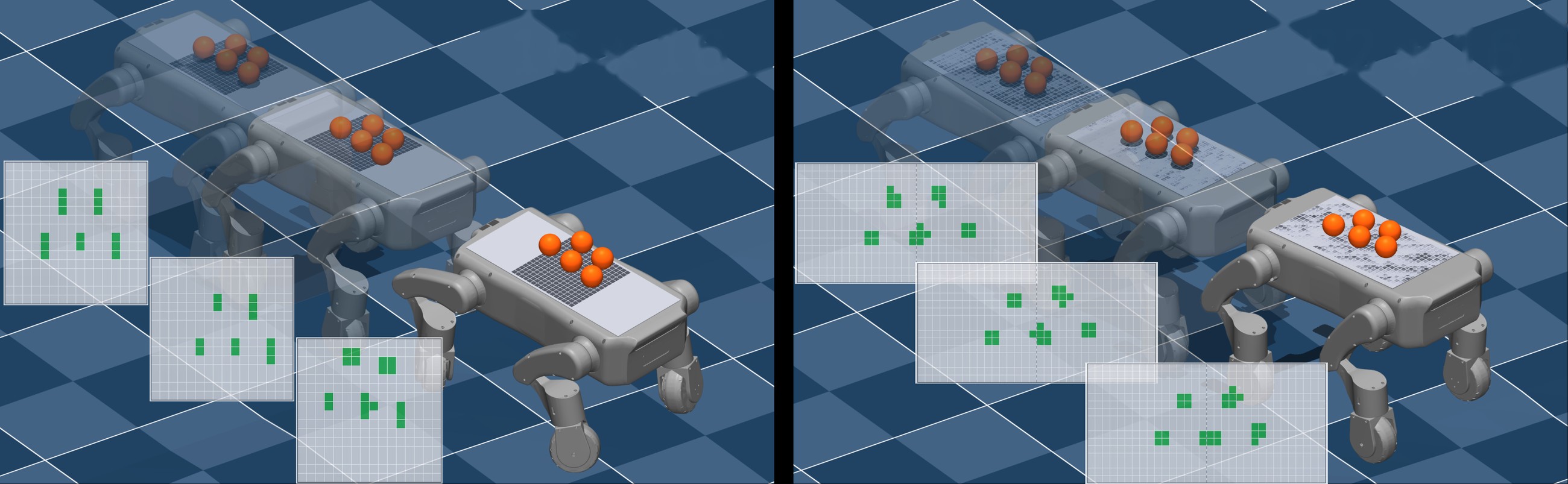}
  \captionof{figure}{Closed-loop MuJoCo rollouts of the distilled student through the deployment ONNX stack on the same robot equipped with twice the tactile sensing area, carrying five spheres. Ghost trails show three timestamps of one rollout; lower panels are the binary contact maps consumed by the student.}
  \label{fig:mujoco}
\end{center}

\end{document}